\documentclass{article}

\usepackage[preprint]{neurips_2018}

\usepackage[utf8]{inputenc} 
\usepackage[T1]{fontenc}    
\usepackage{hyperref}       
\usepackage{url}            
\usepackage{booktabs} 
\usepackage{nicefrac}
\usepackage{microtype}

\usepackage{amsfonts}
\usepackage{amsmath}
\usepackage{amssymb}
\usepackage{amsthm}
\usepackage{graphicx}

\newtheorem{theorem}{Theorem}
\newtheorem{corollary}{Corollary}[theorem]
\newtheorem{lemma}{Lemma}

\usepackage{algorithm}
\usepackage[noend]{algpseudocode}

\title{Variance Regularization for Accelerating Stochastic Optimization}

\author{%
  Tong Yang\thanks{equal contribution} \\
  Department of Physics\\
  Boston College\\
  Newton, MA 02467 \\
  \texttt{yangto@bc.edu} \\
   \And
   Long Sha\footnotemark[1]\\
   Department of Computer Science \\
   Brandeis University \\
   Waltham, MA 02453 \\
   \texttt{longsha@brandeis.edu} \\
   \And
   Pengyu Hong \\
   Department of Computer Science \\
   Brandeis University \\
   Waltham, MA 02453 \\
   \texttt{longsha@brandeis.edu} \\
}

\begin{document}

\maketitle

\begin{abstract}

While nowadays most gradient-based optimization methods focus on exploring the high-dimensional geometric features, the random error accumulated in a stochastic version of any algorithm implementation has not been stressed yet. In this work, we propose a universal principle which reduces the random error accumulation by exploiting statistic information hidden in mini-batch gradients. This is achieved by regularizing the learning-rate according to mini-batch variances. Due to the complementarity of our perspective, this regularization could provide a further improvement for stochastic implementation of generic $1^{st}$ order approaches. With empirical results, we demonstrated the variance regularization could speed up the convergence as well as stabilize the stochastic optimization.

\end{abstract}

\section{INTRODUCTION}

One of the essential tasks in machine learning is optimizing the expectation of a cost function, which is empirically approximated by the average of the cost function over a large set of samples (Tsypkin 1971, Tsypkin 1973, Vapnik 1982, Vapnik 1995):
\begin{align}
    x^* &= \underset{x}{\text{argmin}}\; f(x) \approx \underset{x}{\text{argmin}}\; \frac{1}{N}\sum_{i=1}^{N} \mathrm{f}(x, z_i),
\end{align}
where $z_i$ is a sample, $N$ is the total number of samples, $x$ is the parameter of the cost function $f(x)$, and $\mathrm{f}(x, z_i)$ is the cost of $z_i$ and is differentiable. Gradient descent (GD) algorithms, which rely on $1^{st}$-order gradient, can be used for iteratively solving the above minimization and updating the parameters by moving them in the direction of the negative gradient of the cost. Recently, variants of Stochastic GD (SGD) (Robbins and Monroe 1951) have become dominants in training deep neural networks. SGD approximates the gradient of the cost function at every iteration by sampling a subset of samples.  This research direction turned out to be very effective with many algorithms developed (e.g., Bottou 1998, Schraudolph 1999, Amari, Park, and Fukumizu 2000, Schraudolph 2002, Zhang 2004, Bottou and LeCun 2004, Schraudolph, Yu, and G\"{u}nter 2007, Bottou and Bousquet 2008, Nemirovski, Juditsky, and Shapiro 2009, Bordes, Bottou, and Gallinari 2009, etc.).

Different from conventional $1^{st}$-order techniques that use the same learning rate to update all parameters, AdaGrad (Duchi et al. 2011) deploys a data-dependent way to automatically derive the learning rate on a per-parameter basis. AdaGrad achieves this by dividing the gradient of the current iteration by the square root of the sum of the squared gradients in the history. It offers two variants. One is based on Composite Mirror Descent (Duchi et al. 2010), and the other is based on Xiao's 2010 (Xiao 2010) extension of Nesterov's primal-dual subgradient method (Nesterov 2009).  The major contribution of AdaGrad is that it adds a data-dependent proximal term which depends on the geometry of the data from all past iterations.  Previously in the composite mirror descent and primal dual sub-gradient methods, proximal terms were used but were either fixed or multiplied by a time-dependent constant.  Making the proximal term data-dependent makes it possible to derive tighter bounds on the regret. However, the accumulation of the squared gradients in Adagrad keeps growing during training and causes the learning rate to shrink and eventually become infinitesimally small so that the algorithm is no longer able to learn from data.

To mitigate the rapid decay of the learning rate in AdaGrad, RMSprop (Tieleman and Hinton, 2012), Adam (Kingma and Ba 2015), AdaDelta (Zeiler, 2012), and Nadam (Dozat, 2016) were proposed to limit the influence of historical gradients by using the exponential moving averages of squared gradients in the past. Although this strategy works well in some deep learning applications, these algorithms failed to outperform SGD with momentum. AMSGrad (Reddi et al. 2018)) elucidates the problem of using the exponential moving average of past squared gradients, and proposes to use the maximum of past squared gradients instead. AMSGrad delivers improved performances over Adam on a few popular tasks.

In addition to above $1^{st}$-order methods, another essential issue is the stochasticity in mini-batch approaches. It has been noticed in the previous works on SVRG (Johnson and Zhang 2013), where the stochastic error is reduced at the cost of a calculation of across the whole batch at scheduled moments, and the estimation of error terms is based not on present but delayed knowledge. While it remains unclear to solve the problem with an universal design for general $1^{st}$-order algorithms, we realize that there is "treasure" hidden in the gradients of mini-batches which can be explored to achieve the goal. To this end, we propose a new universal principle to facilitate all $1^{st}$-order algorithms. Our principle utilizes the variance of gradients within each mini-batch to modulate the per-parameter learning rate, and hence is named \underline{V}ariance-\underline{R}egularization (\textbf{VR}). In contrast to existing $1^{st}$-order designs, VR focuses on exploring the stochasticity brought in by generic stochastic implementations other than the high-dimensional geometry and hence can be used as a supplementary to accelerate all existing mini-batch methods. The intuition behind VR is that the update of parameters should be more cautious if there is a relatively larger variance among samples, which can be related to the uncertainty in a stochastic updating rule. We discussed this principle in details from a theoretical perspective and demonstrated the improvement by applying the proposed regularization on SGD. All proofs can be found in \textbf{Supplementary Material.}


\section{STOCHASTICITY IN GRADIENT BASED METHODS}

Before introducing our new principle, we analyze the impact of stochasticity in $1^{st}$-order methods. This requires a more thorough analysis of the convergence when randomness iterations are involved. As a demonstration, we would compare only GD and SGD below, which could be easily generalized to the comparison between batch/mini-batch versions of other $1^{st}$-order methods. 
A parameter to be updated in $t$-th iterations is denoted as $x_t$, and its corresponding gradient calculated with respect to sample $i$ is denoted as $\mathrm{g}_{t,i}$. The value of the cost-function $f(x)$ at $x_{t}$ is denoted as $f_t$. And we assume there exits an optimal value $x^*$ minimizing the cost function to $f^* \triangleq f(x^*)$.

\paragraph{GD -- A Deterministic Process}
We define the following term calculated in $t$-th iteration with respect to the whole data-set (batch method) as the \textit{true-gradient}:
\begin{align}\label{g_t}
g_t = \frac{1}{N}\sum_{i=1}^{N} \mathrm{g}_{t,i}.
\end{align}
In a GD method, the updating rule is:
\begin{align}\label{GD}
x_{t+1} &=  x_t - \gamma_t g_t,
\end{align}
where $\gamma_t$ in general is a time-dependent learning rate. Given $x_t$ (and $\gamma_t$), the true gradient $g_t$ and hence the next-moment parameter $x_{t+1}$ are fully determined without any randomness.
Therefore this is a deterministic process. 

\paragraph{SGD -- A Stochastic Process}
We define the following term calculated in $t$-th iteration with respect to only a mini-batch $\{Z_t\}$ as the \textit{mean-gradient}:
\begin{align}\label{mean_g}
\bar{g}_t = \frac{1}{m}\sum_{i\in Z_t}^{m} \mathrm{g}_{t,i}.
\end{align}
And the stochastic version of the updating rule is:
\begin{align}\label{SGD}
x_{t+1} &= x_t - \gamma_t \bar{g}_t.
\end{align}
Since $\{Z_t\}$ is a mini-batch randomly selected from the whole dataset which brings in certain stochasticity, it is legal to view both $x_t$ and $\bar{g}_t$ as random variables. We define the past history information set: $\mathcal{P}_t \triangleq \{x_{\tau}\}_{\tau=0}^t.$

\paragraph{Relation between GD/SGD}  
Generally from GD and SGD, $\{g_t\}$ and $\{\bar{g}_t\}$ are two completely different sets. To make comparisons, we would assume: at any coordinate $x_t$ in the parameter space, the one-sample gradients $\mathrm{g}_{t}$ are $iid$ variables, and that the $1^{st}$ and $2^{nd}$ order moments of $\mathrm{g}_{t}$ obeys:
\begin{align}
\mathbb{E}\big[\mathrm{g}_t\big|\mathcal{P}_t\big] = g_t, \qquad E\big[\big(\mathrm{g}_t^2 - g^2_t\big)\big|\mathcal{P}_t\big] = \sigma_{0}^2(x_t).
\end{align}
The notation $g_t$ is same as the true-gradient -- this is not an abuse but a further assumption for comparison between GD/SGD. Besides, we also assume $\mathrm{g}_t$ at different moments are independent to each other.

According to the $iid$ nature, the Central Limit Theorem (CLT) implies that the statistics of mean-gradient $\bar{g}_t$, which is also a random variable, should obey:
\begin{align}
\mathbb{E}\big[\bar{g}_t\big|\mathcal{P}_t\big] = g_t, \qquad \mathbb{E}\big[\big(\bar{g}_t^2 - g^2_t\big)\big|\mathcal{P}_t\big] = \frac{\sigma_{0}^2(x_t)}{m}.
\end{align}
In the limit $m\rightarrow \infty$, we have:
\begin{align}\label{limit}
\lim_{m\rightarrow \infty}\bar{g}_t = g_t,\qquad \lim_{m\rightarrow \infty}\sigma_{t}^2 = 0,
\end{align}
and the method goes back to the GD version. 

The convergence of GD has been well-studied. In the case of SGD, the convergence is usually discussed by analyzing the expectation value of certain criteria, such as Lyapunov-criteria $|x_{t}-x^*|^2$ (Bottou 1998), or the accumulated loss $\sum^T_{t=0} \big[f_{t+1} - f^*\big]$ (Shamir and Zhang 2013). There are, however, two more questions requiring further investigations:
\begin{enumerate}
\item As a stochastic process, the convergence of expectation values is not enough. And the convergence of variances is also necessary. If the variance diverges, then in practice an algorithm might still behave poorly.

\item Within a permitted range, a larger learning rate accelerates the convergence of mean value, but may also enlarge the error accumulation from stochasticity. Hence a trade-off has to be made.
\end{enumerate}

The convergence of the variance has been analyzed in \textbf{Supplementary Material S1.1}. In the following, we analyze the impact of the stochastic error on the convergence rate. 

\subsection{Impact of Stochasticity on Convergence}

As used most in current literature, we choose the following criteria, which is an averaged loss:
\begin{align}
\frac{1}{T}\sum_{t=1}^T\big[f(x_{t}) - f^*\big].
\end{align}
The convergence of this criteria has been studied thoroughly (Shamir and Zhang 2013). While in our current analysis, we would focus on the impact of the variance term which captures the stochasticity. To simplify the discussion, we assume the $L$-Lipschitz continuity for the convex loss function $f(x)$, i.e. $\|g(x) - g(y)\| \leq L\|x-y\|_2$. And we would further assume $\gamma_t \leq \frac{1}{L}$ at any moment $t$ to make the learning process more monotonic, i.e. loss function keeps decreasing. It is reasonable to put these strong constraints since we are only interested in the role of stochasticity. For convenience, we define the following mean-zero random variable:
\begin{align}
\epsilon_t \triangleq \bar{g}_t - g_t.
\end{align}
We also define $\sigma_t^2\triangleq \mathbb{E}[\epsilon_t^2]$ and $\eta_t^4\triangleq \mathbb{E}[\epsilon_t^4]$, and the normal-time-average of them: $\bar{\sigma}^2\triangleq \frac{\sum_t \sigma_t^2}{T}$  and $\bar{\eta}^4 \triangleq \frac{1}{T}\sum_t \eta_t^4$, which are assumed to be finite. 
Firstly, we could derive an upper bound for the criteria:
\begin{lemma}\label{lem_los_upbnd}
Assume the $L$-Lipschitz continuity for the convex loss function $f(x)$, the learning rate $\gamma_t \leq \frac{1}{L}$ at any time $t$, and a bounded region for the parameter $|x_t-x^*|^2 \leq M^2$. Then the criteria $\frac{1}{T}\sum_{t=1}^{T}\big[f(x_{t}) - f^*\big]$ would be suppressed below the following upper bound:
\begin{align}\label{upbnd}
S_T =&  \frac{1}{\gamma_TT}\cdot\frac{M^2}{2}+ \frac{1}{2T}\sum_{t=1}^{T}\gamma_t\big(1+L\gamma_t\big)\epsilon_t^2\nonumber \\
&- \frac{1}{T}\sum_{t=1}^{T} \epsilon_t\big(x_t-L\gamma_t^2g_t-x^*\big).
\end{align}
\end{lemma}
With \textbf{Lemma \ref{lem_los_upbnd}}, we could obtained the following results for the convergence of upper bound $S_T$ in \eqref{upbnd}:
\begin{theorem}\label{them1_conv_lr}
Assume the $L$-Lipschitz continuity for the convex loss function $f(x)$, the learning rate $\gamma_t \leq \frac{1}{L}$ at any time $t$, and a bounded region for the parameter $|x_t-x^*|^2 \leq M^2$. Further assume a finite upper bound exists such that $\sigma_0^2\geq \sigma_t^2$ for any time $t$. Then the expectation of upper bound $S_T$ \eqref{upbnd} satisfies:
\begin{align}\label{S2}
\mathbb{E}[S_T] \leq \frac{1}{\gamma_TT}\cdot\frac{M^2}{2} + \frac{\sum_t \gamma_t}{T}\cdot \sigma_0^2.
\end{align}
\end{theorem}
\begin{theorem}\label{them2_conv_lr}
Assume the $L$-Lipschitz continuity for the convex loss function $f(x)$, the learning rate $\gamma_t \leq \frac{1}{L}$ at any moment, and a bounded region for the parameter $|x_t-x^*|^2 \leq M^2$; and for simplicity, assume the distribution of $\epsilon_t$ at any moment $t$ to be symmetric, then the following result holds:
\begin{align}
Var(S_T) \leq \frac{1}{T}\bigg[4M^2\bar{\sigma}^2 + \frac{\bar{\eta^4}}{L^2}\bigg].
\end{align}
\end{theorem}
Now suppose one designs a vanishing learning rate:
\begin{align}
\gamma_t \sim t^{-\alpha}.
\end{align}
From the result \eqref{S2}, the expectation value behaves\footnote{See Supplementary Material S1.3 for explanations.}:
\begin{align}
\mathbb{E}[S_T] &\sim \mathcal{O}(T^{\alpha -1}) + \mathcal{O}(T^{-\alpha}).
\end{align}
Again, due to the stochasticity, at most we could achieve a convergence of order $\frac{1}{\sqrt{T}}$ by choosing $\alpha = \frac{1}{2}$.

The analysis above indicates the existence of stochasticity requires a trade-off between a faster convergence of mean value and a smaller accumulation of error. From \eqref{S2}, we know the accumulated loss at least would scale as
$\mathbb{E}[S_T] \sim \frac{M\sigma_0}{\sqrt{T}}$. However, the above analysis assumes the learning rate design is independent of the variance term $\sigma_t^2$. If we can adjust the learning rate $\gamma_t$ according to $\sigma_t^2$, then the accumulated loss might be effectively reduced.

\section{VARIANCE REGULARIZATION OF LEARNING RATE}

In this section we would derive a variance-based regularization of learning rates to reduce the impact of stochastic error. 
Importantly, this adaptive regularization is independent of any geometry-adaptive design, such as Adam, AdaGrad, or RMSProp.

From now on, we consider a general updating rule:
\begin{align}\label{general_update}
x_{t+1} &= x_t - \gamma_t \bar{\delta}_t.
\end{align}
Note that here we change notations, and $\bar{\delta}_t$ is NOT restricted to mean-gradient anymore, but can be any updating increment of parameter $x_t$. For example, in Adam, $\bar{\delta}_t$ would represent a "history-normalized momentum". Other definitions, such as $\epsilon_t$, $\sigma_t^2$, and so on,  would be modified correspondingly. And we can separate the whole learning rate into two parts:
\begin{align}
\gamma_t \equiv \alpha_t\lambda_t \triangleq \alpha_t\lambda[\sigma_t^2].
\end{align}
The new variance-based regularization acts on the $\lambda_t$ part, which should be of $\mathcal{O}(1)$. And the case $\lambda_t\equiv 1$ corresponds to conventional methods. We would call $\lambda_t$ the \textit{regularizer}.

\subsection{Intuition and Caveats}

The intuition of this adaptive regularization is simple:  the variance captures the deviation of mean-gradients from true-gradients and hence is related to the probability to go in a wrong direction. Therefore a natural strategy follows: when the risk of mistaking is larger, take smaller steps to prevent large deviations from the track. As a result, we would try to reduce the learning-rate (hence learning rate) when the variance is large.

There are, however, some caveats to be kept in mind:
\begin{enumerate}
\item The smaller the learning rate, the closer the mini-batch algorithm stays to the track of the batch-version. However, a finite learning rate is required to ensure the basic convergence of the mean value. Mathematically, when learning-rates are too small, the first term in \eqref{S2} would diverge.
\item In the derivation of the above upper bound $S_T$, we assumed $\gamma_t \leq \frac{1}{L}$. Although this is not required strictly by the problem itself,  to simplify the discussion, we would still consider this constraint in the calculation.
\end{enumerate}

\subsection{Mathematical Formulation of Optimizing Stochasticity}

Now we formulate the problem into a mathematical optimization problem. Basically, we are trying to minimize the following term:
\begin{align}\label{p0}
\sum_{t=1}^{T}\gamma_t\big[1+L\gamma_t\big]\sigma_t^2,
\end{align}
by varying the non-negative values of $\{\gamma_t\}$, \emph{while ensuring the convergence of $\mathbb{E}[S_T]$ at the same time}. For simplicity we assume $\alpha_t \equiv \alpha_0 \lesssim \frac{1}{L}$, and define: $\lambda_0 \triangleq  \frac{1}{\alpha_0L}$.
To ensure a monotonic decrease of cost function, we should require $\lambda_t \leq \lambda_0$. Then the above problem becomes to minimize:
\begin{align}\label{p1}
Q_T \triangleq \sum_{t=1}^{T}\lambda_t\bigg[1+\frac{\lambda_t}{\lambda_0}\bigg]\sigma_t^2.
\end{align}
Without extra constraint, this problem has the unique solution at $\lambda_t \equiv 0$ , which is not surprising since $Q_T$ captures the deviation from a batch-method, and the deviation would be zero if no step is taken at all.
To prevent this result, we add a lower bound $\lambda_{min}$ to ensure a non-zero learning-rate at each iteration. Besides, as mentioned above, we need to ensure the convergence of $\mathbb{E}[S_T]$. Therefore, we add an averaged global constraint to ensure the mean-value convergence:
\begin{align}\label{p3}
\sum_{t=1}^T \lambda_t = T\bar{\lambda}.
\end{align}
And for simplicity, we set $\bar{\lambda}=1$.
However, the solution may still not be valid: it would assign an extremely large learning-rate to the moment with minimum $\sigma_t^2$ while leaving others to be small. This means the algorithm moves very slowly except jumps at the extreme moment, which is unreasonable. Hence, we require also an upper bound $\lambda_{max}$ for regularizers \footnote{e.g., one may directly use $\lambda_0$ as the upper bound.}.

Combine the problem \eqref{p1} with all constraints, we eventually formulate the problem into:
\begin{equation}
\left\{\arraycolsep=1.4pt\def\arraystretch{2.2}
\begin{array}{l}
\lambda_t = \underset{\{\lambda_t\}}{\min}\bigg\{\sum_t\lambda_t\big[1+\frac{\lambda_t}{\lambda_0}\big]\sigma_t^2  \:\bigg|\: \{\sigma_t^2\},L \bigg\}, \\
\sum_t \lambda_t =T, \qquad \lambda_t \in [ \lambda_{min}, \lambda_{max}].
\end{array}
\right.
\end{equation}
This is a problem to minimize the stochastic error when ensuring the convergence of $\mathbb{E}[S_T]$.

\subsection{Approximate Solutions}

The above problem is not trivial to solve. 
And more importantly, the solution might include a relation between the variants at two arbitrary moments $\sigma_{t}^2$ and $\sigma_{t'}^2$, while in practice one can at most know the past history (the causality). Therefore, we want an approximate solution obeying the causality, i.e. to find $\{\lambda_t\}$ that depend only on the given history $\{\sigma_{\tau}^2|\tau\leq t\}$.

Now suppose $\{\sigma_t^2\}$ values have small variations, then the optimized $\{\lambda_t\}$ should also be close to each other, in which case the second constraint could be automatically satisfied after solving the following problem:
\begin{equation}\label{approx_prob}
\left\{\arraycolsep=1.0pt\def\arraystretch{1.8}
\begin{array}{l}
\lambda_t = \underset{\{\lambda_t\}}{\min}\bigg\{\sum_t\lambda_t\big[1+\frac{\lambda_t}{\lambda_0}\big]\sigma_t^2  \:\bigg|\: \{\sigma_t^2\},L \bigg\}, \\
\sum_t \lambda_t =T,
\end{array}
\right.
\end{equation}
which is an approximate problem. We the obtained:
\begin{theorem}\label{them_solution}
Assume the $L$-Lipschitz continuity for the convex loss function $f(x)$, and that $\alpha < \frac{1}{L}$; further assume that  local variances $\sigma_t^2$ at two arbitrary different moments satisfy:
 \begin{align}\label{assumption}
\max_{s,t}\big(|\sigma_t^2 - \sigma_s^2|\big) \ll \bar{\sigma}^2.
\end{align}
Then the stochastic error accumulation \eqref{p0} could be minimized by the following regularizer:
\begin{align}\label{sl1}
\lambda_t = \bigg(1 + \frac{\lambda_0}{2}\bigg)\frac{\bar{\sigma}^2}{\sigma_t^2} - \frac{\lambda_0}{2},
\end{align}
when the convergence of $\mathbb{E}[S_T]$ is still ensured.
\end{theorem}
Note the solution \eqref{sl1} reduces to $\lambda_t\equiv1$ if $\sigma_t^2\equiv \bar{\sigma}^2$, which suggests that the proposed regularization improves the algorithm only in the existence of heteroskedasticity in the time series $\{\bar{\delta}_t\}$.

The above solution, however, is still not applicable to practical implementations. Firstly, it is only valid under the assumption \eqref{assumption}.
Secondly, the $\bar{\sigma}^2$ term in the solution still violates the causality.

\section{VR Algorithms}

To find a practically applicable solution, we need a bounded solution of $\lambda_t$ for generic cases instead of \eqref{sl1} obtained given \eqref{assumption}. The form of \eqref{sl1} indicates that a regularizer should be related to the ratio $\frac{\bar{\sigma}^2}{\sigma_t^2}$, and should reduce to $\lambda_t \equiv 1$ in the case $\sigma_t^2\equiv \bar{\sigma}^2$. Without the assumption \eqref{assumption}, the ratio $\frac{\bar{\sigma}^2}{\sigma_t^2}$ behaves poorly, and $\lambda_t$ does not automatically satisfy bounded-region constraint. To enforce this constraint, we consider the following bounded-form\footnote{See Supplementary Material S2.1 for discussions.}:
\begin{align}\label{sl2}
\lambda_t  = \frac{1 + s}{1 + s \cdot \frac{\sigma_t^2}{\bar{\sigma}^2}},
\end{align}
with certain \textit{impact factor}\footnote{See Supplementary Material S2.1 for discussions.} $s>0$, which in our experiments we set as $s=2$. Then in general:
\begin{align}
 \lambda_t \in \bigg[\frac{1 + s}{1 + s \cdot \frac{\sigma_0^2}{\bar{\sigma}^2}}, \;\; (1+s)\bigg],
\end{align}
where $\sigma_0^2$ represents an upper bound of $\sigma_t^2$. This form has the natural upper and lower bounds. Besides, in the case $\sigma_t^2\equiv \bar{\sigma}^2$, we constantly have $\lambda_t =1$. 

The analysis above suggests the form \eqref{sl2} is an applicable regularizer for generic cases. We would use this factor to design specific algorithms. To get rid of the causality issue, we would use the history-average $\bar{\sigma}^2_t = \frac{1}{t}\sum_{\tau}\sigma_{\tau}^2$ as an approximation of $\bar{\sigma}^2$. Besides, we would use the mini-batch variance $v_t^2$ and $\bar{v}_t^2$ to replace $\sigma_t^2$ and $\bar{\sigma}_t^2$ in practice\footnote{See Supplementary Material S2.2 for explanations.}, which is defined as:
\begin{align}
v_t^2 \triangleq \frac{1}{m}\sum_i (\delta_{i,t} - \bar{\delta}_t)^2,\qquad \bar{v}_t^2 \triangleq \frac{1}{t}\sum_{\tau =1}^{t} v_{\tau}^2
\end{align}
Moreover, since $v_t^2$ scales with gradients, in practice we would use a \textit{scale-free variance}: $\rho_t$, which could better capture the stochastic uncertainty. For example, the variance would be always small when gradients themselves are small; however, the relative uncertainty mentioned in our intuition could still be quite large\footnote{See Supplementary Material S2.3 for discussions.}.

To demonstrate the proposed regularizer, we designed below a \textit{regularized} version of SGD (VR-SGD).

\subsection{VR for SGD}
In SGD, $\bar{\delta}_t$ is directly the mean-gradient \textbf{$\bar{\delta}_t =\bar{g}_t$}, and $\rho_t$ is:
\begin{align}\label{rho_sgd}
\rho_t \triangleq \frac{v_t^2 }{\bar{\delta}_t^2},
\end{align}
And the derived algorithm is:
\begin{algorithm}
\caption{VR-SGD}\label{vrg_sgd}
\begin{algorithmic}[1]
\State{$\alpha>0,\,  s>0$}
\State $ t=0, \,  \Omega_0 = 0$
\While {$x_t$ not converge}
\State{$t \gets t+1$}
\State{Acquire a $m$-samples mini-batch $\{z_{t,i}\}$}
\State{$\bar{\delta}_t \gets \frac{1}{m}\sum_i \nabla_x \mathrm{f}(x_t, z_{t,i})$}
\State{$\rho_t \gets \frac{1}{m\bar{\delta}_t^2}\sum_i \nabla_x \mathrm{f}(x_t, z_{t,i})^2 - 1$}
\State{$\Omega_t \gets \Omega_{t-1} + \rho_t $}
\State{$x_t \gets x_{t-1} - \alpha\big[\frac{1 + s}{1 + s\cdot\frac{\rho_t}{\Omega_t}}\big]\bar{\delta}_t $}
\EndWhile
\Return {$x_t$}
\end{algorithmic}
\end{algorithm}

\section{EXPERIMENTS}

We have explored the proposed optimization algorithms on two popular benchmark datasets: \textbf{Fashion-MNIST} (Xiao, Rasul and Vollgraf, 2017) and \textbf{CIFAR-10} (Krizhevsky and Hinton, 2009). 

All the experiments are implemented using Pytorch (Paszke etc., 2017). As the current auto-differentiate framework of Pytorch does not provide gradient variances or per-sample gradients from batch gradient back-propagation, we implemented a procedure which calculates the gradients of individual samples and then the variance within each mini-batch on the fly. This implementation is not efficient. The performance of our method can be much more efficiently improved if the auto-differentiate framework can calculate per-batch gradient variances. We run all the experiments on our machine with 4 NVIDIA Titan-X Pascal GPUs.

\subsection{Fashion-MNIST Experiments}

\textbf{Fashion-MNIST} is a dataset composed of $28\times28$ grey scale images of pieces of 10-classes of clothing. The objective is to correctly classify the class of a specific image. We normalize the examples to unitary-norm before experiments. The dataset is divided into 60,000 training examples and 10,000 test examples.

For this experiment, a simple convolution neural network model is applied: we use two sequential convolution-pooling-activation blocks then appended 2 fully connected layers and a softmax layer.
In SGD versus VR-SGD, we fix the learning rate at 0.01. A batch size of 100 is used. We monitor both training process until both algorithms converges. We observe from Figure.\ref{fig:fmnist_sgd} that, by implementing the proposed regularization, the training process speeds-up remarkably as both training loss and accuracy outperform the SGD without regularization.

\begin{figure}[ht!]
    \centering
    \includegraphics[width=0.7\textwidth]{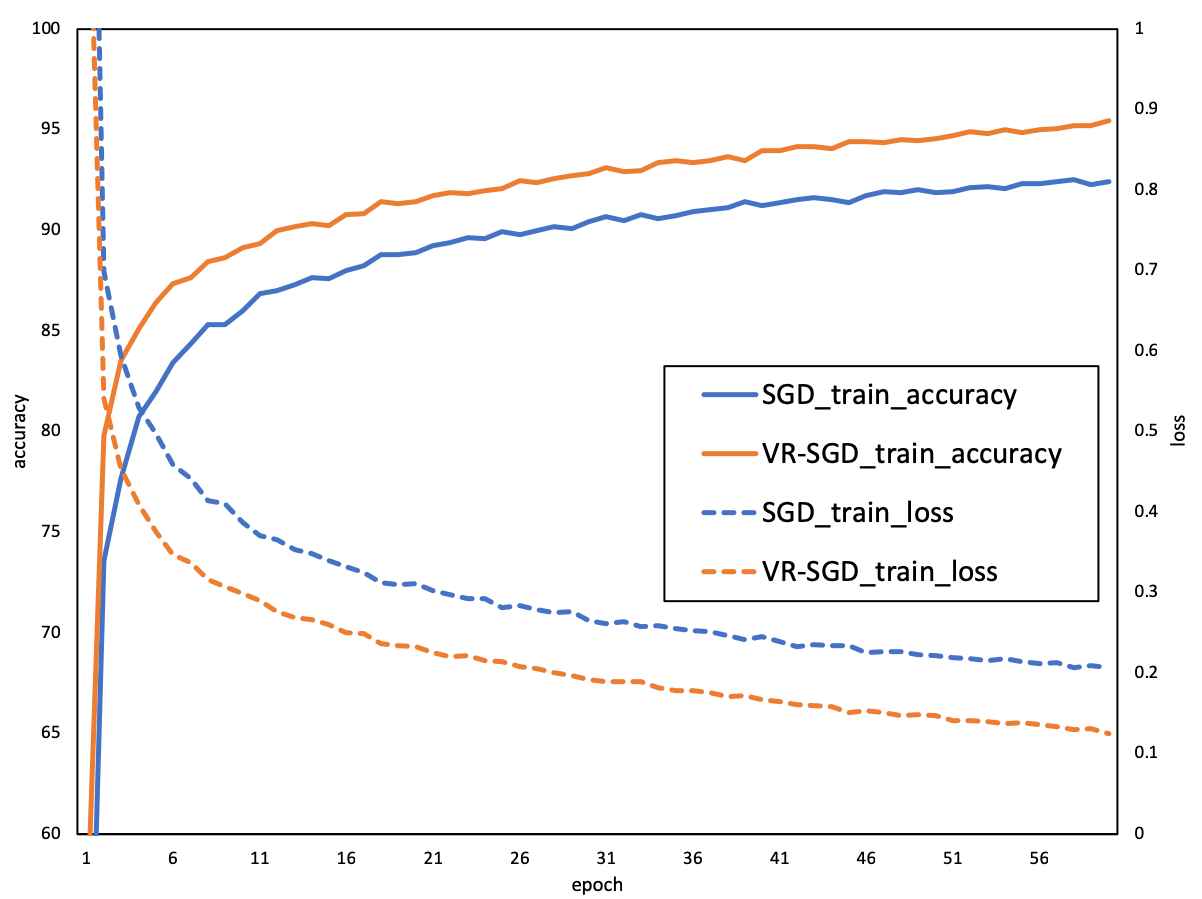}
    \vspace{.1in}
    \caption{Training loss and accuracy comparison of VR-SGD and SGD on Fashion-MNIST.}
\label{fig:fmnist_sgd}
\end{figure}


\subsection{CIFAR-10 Experiments}

The second part of our experiments is done on the \textbf{CIFAR-10} dataset which comprises $50,000$ training, $10,000$ testing images, each having a 32x32 resolution and 3 color channels. The goal is to correctly classify the following ten categories: airplane, automobile, bird, cat, deer, dog, frog, horse, ship, truck. We apply the normalization in this dataset with the same fashion in the previous experiment. We use \textit{LeNet}(LeCun 1998) as our model, consisting of two convolution blocks and 3 fully connected layers.

\begin{figure}[ht!]
    \centering
    \includegraphics[width=0.7\textwidth]{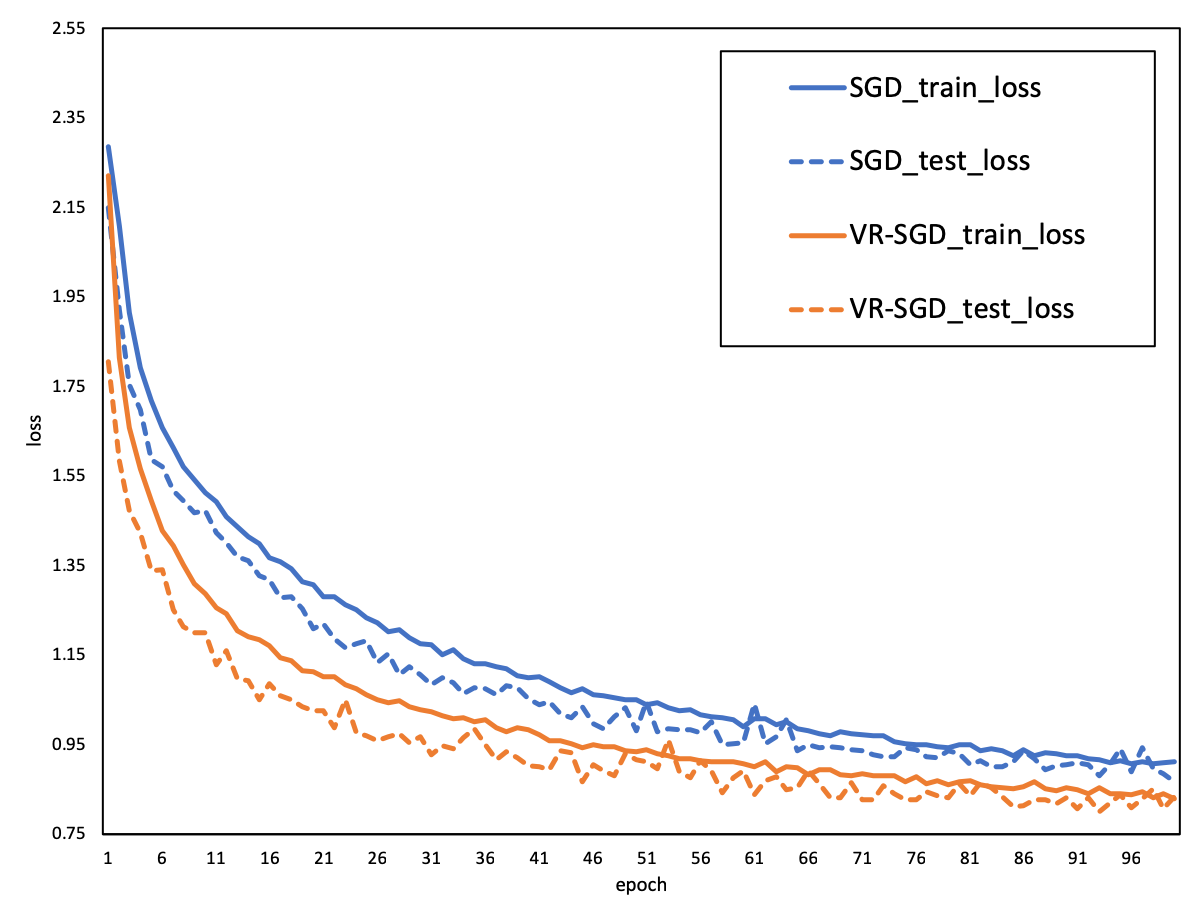}
    \vspace{.1in}
    \caption{Training and testing loss comparison of VR-SGD and SGD on CIFAR-10.}
\label{fig:c10_loss}
\end{figure}

\begin{figure}[ht!]
    \centering
    \includegraphics[width=0.7\textwidth]{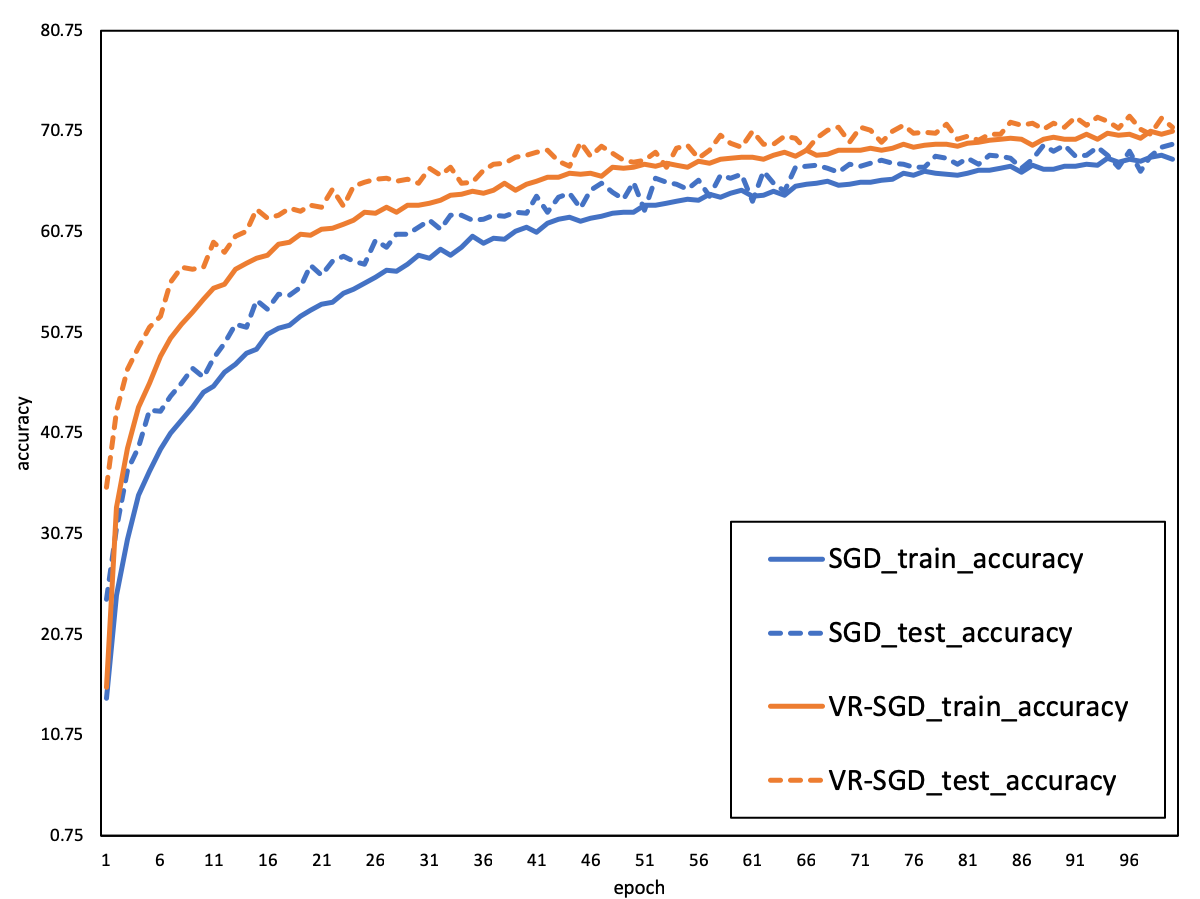}
    \vspace{.1in}
    \caption{Training and testing accuracy comparison of VR-SGD and SGD on CIFAR-10.}
\label{fig:c10_acc}
\end{figure}

Similar patterns, as in \textbf{Fashion-MNIST} case, are observed in Figure. \ref{fig:c10_loss} and Figure. \ref{fig:c10_acc}. Our modification boosted the convergence, especially during the first training steps. We used the default PyTorch setting with learning rate at $0.01$ on both experiments. The variance regularization aided the algorithm, setting a dynamical update and speeding-up the training towards the optima.

\section{CONCLUSION AND DISCUSSION}
In this work, we propose a universal principle which reduces the random error accumulation in current first-order methods, by regularizing the learning according to the variance of mini-batches. This regularization focuses on the heteroskedasticity in a generic stochastic training process, and hence can be implemented easily on popular optimization algorithms to obtain supplementary improvements. We derived the form of regularizer from a theoretical analysis and demonstrated empirically that our regularization could speed up the convergence.

Besides the results mentioned in the main text, in practice, the impact-factor $s$ controls the "magnitude" of regularization. In the above algorithms, $s$ is set as a constant. While on the one hand, regularization reduces the error accumulation from stochasticity; on the other hand, for complicated geometries, the existence of stochasticity would help the algorithm to jump out of local minima. And we could take advantage of stochasticity in the later time when the algorithm is trapped in local minima, by properly defining a decreasing impact-factor $s_t$, which would bring further improvement.

One should be aware that the proposed regularization is derived from the upper-bound behavior
\eqref{S2} other than the average-cost itself, and also that the analysis is restricted to the $2^{nd}$-order statistics, which only captures a small amount of information on the distribution of random variables in a stochastic algorithm. A natural extension would be including more statistic information by considering the sophisticated regularization mechanism. In addition, mini-batch statistics actually may even contain more fruitful information than the stochasticity considered in the current context: the stochasticity disappears in a batch method, while the sample-statistics stills play an important role in formulating the cost-function. A more general formulation would be explored in later work. 

\section*{References}

\small

[1] S. Amari, H. Park, and K. Fukumizu. Adaptive Method of Realizing Natural Gradient Learning for Multilayer Perceptrons.\textit{Neural Computation}, Volume: 12, Issue: 6, 1399-1409, 2000.
	
[2] A. Borders, L. Bottou, and P. Gallinari. SGD-QN: Careful Quasi-Newton Stochastic Gradient Descent. \textit{The Journal of Machine Learning Research}, Volume: 10, 1737-1754, 2009

[3] L. Bottou. Online Learning and Stochastic Approximations. In D. Saad (ed.), \textit{On-line learning in neural networks}, 17 (9), 142, 1998.
	
[4] L. Bottou and O. Bousquet. The Tradeoffs of Large Scale Learning. \textit{Advances in Neural Information Processing Systems}, Volume: 20, 161-168, 2008.
	
[5] L. Bottou and Y. LeCun. Large scale online learning. \textit{Advances in Neural Information Processing Systems}, 16, MIT Press, 2004.

[6] T. Dozat. Incorporating Nesterov Momentum into Adam. \textit{ICLR}, 2016
	
[7] J. Duchi, E. Hazan, Y. Singer. Adaptive Subgradient Methods for Online Learning and Stochastic Optimization. \textit{Journal of Machine Learning Research}, Volume: 12, 2121-2159, 2011.
	
[8] J. Duchi, S. Shalev-Shwartz, Y. Singer, and A. Tewari. Composite Objective Mirror Descent.
	\newblock \textit{COLT}, 2010.
	
[9] R. Johnson and T. Zhang. Accelerating stochastic gradient descent using predictive variance reduction. \textit{International Conference on Neural Information Processing Systems}, Volume: 1, 315-323, 2013. 
	
[10] D. P. Kingman and J. L. Ba. Adam: A Method for Stochastic Optimization. \textit{ICLR}, 2015.

[11] A. Krizhevsky and G. Hinton. Learning multiple layers of features from tiny images. \text{Technical report}, University of Toronto, 1 (4), 7, 2009.
    
[12] Y. LeCun, L. Bottou, Y. Bengio and P. Haffner. Gradient-based learning applied to document recognition. \textit{Proceedings of the IEEE}, Vol: 86, No. 11, 2278-2324, 1998.

[13] A. Nemirovski, A. Juditsky, G. Lan, and A. Shapiro. Robust Stochastic Approximation Approach to Stochastic Programming.	\textit{SIAM J. Optim.}, 19 (4), 1574–1609, 2009.
	
[14] Y. Nesterov. Primal-dual subgradient methods for convex problems. \textit{Mathematical Programming}, Volume: 120, Issue: 1, 221–259, 2009.
	
[15] A. Paszke, S. Gross, S. Chintala, G. Chanan, E. Yang, Z. DeVito, Z. Lin, A. Desmaison, L. Antiga, and A. Lerer. Automatic differentiation in PyTorch. \textit{NIPS 2017 Workshop Autodiff Program Chairs}, 2017.

[16] S. J. Reddi, S. Kale, S. Kumar. On the Convergence of Adam and Beyond. \textit{ICLR}, 2018.

[17] H. Robbins and S. Monro. A Stochastic Approximation Method. \textit{Ann. Math. Statist.} Volume: 22, Number: 3, 400-407, 1951.	
	
[18] N. N. Schraudolph. Local gain adaptation in stochastic gradient descent. \textit{Ninth International Conference on Artificial Neural Networks}, Volume: 2, 569-574,1999.
	
[19] N. N. Schraudolph. Fast Curvature Matrix-Vector Products for Second-Order Gradient Descent. \textit{Neural Computation}, Volume: 14, Issue: 7, 1723-1738, 2002.
	
[20] N. N. Schraudolph, J. Yu, and S. G\"{u}nter. A Stochastic Quasi-Newton Method for Online Convex Optimization. \textit{Proceedings of the Eleventh International Conference on Artificial Intelligence and Statistics}, PMLR 2:436-443, 2007.
	
[21] O. Shamir and T. Zhang. Stochastic Gradient Descent for Non-smooth Optimization: Convergence Results and Optimal Averaging Schemes. \textit{ICML}, 2013.
	
[22] T. Tieleman and G. Hinton. Lecture 6.5-rmsprop: Divide the gradient by a running average of its recent magnitude. \textit{COURSERA: Neural networks for machine learning}, 4 (2), 26-31, 2012.

[23] Y. Z. Tsypkin. \textit{Adaptation and Learning in Automatic Systems.} New York: Academic Press, 1971.

[24] Y. Z. Tsypkin. \textit{Foundations of the Theory of Learning Systems.} New York: Academic Press, 1973.
	
[25] V. Vapnik. \textit{Estimation of Dependencies Based on Empirical Data.} New York: Springer-Verlag, 1982.
	
[26] V. Vapnik. \textit{The Nature of Statistical Learning Theory.} New York: Springer-Verlag, 1995.

[27] H. Xiao, K. Rasul, and R. Vollgraf.  Fashion-MNIST: a novel image dataset for benchmarking machine learning algorithms.  \textit{arXiv}: 1708.07747, 2017.

[28] L. Xiao. Dual Averaging Methods for Regularized Stochastic Learning and Online Optimization. \textit{Journal of Machine Learning Research}, Volume: 12, 2543-2596, 2010.

[29] M. D. Zeiler. ADADELTA: An Adaptive Learning Rate Method. \textit{arXiv: 1212.5701}, 2012.

\newpage 

\newcommand{\beginsupplement}{%
		\onecolumn%
        \setcounter{section}{0}%
        \renewcommand{\thesection}{S\arabic{section}}%
        \setcounter{equation}{0}%
        \renewcommand{\theequation}{S\arabic{equation}}%
        \setcounter{theorem}{-1}%
        \setcounter{lemma}{0}%
     }

\beginsupplement     
\begin{center}
\textbf{\Large Supplemental Materials}
\end{center}

\section{Proof of theorems in the main text}
In this section we provide proofs of all lemmas and theorems mentioned in the main text.

\subsection{Convergence of variance}\label{S1.1}

Here we prove the convergence of the following Lyapunov criteria in a more rigorous way:
\begin{align}
R_t \triangleq  |x_t - x^*|^2.
\end{align}
According to the updating rule, we have the following iterating relation:
\begin{align}\label{S_ite_R}
R_{t+1}=R_{t} - 2\gamma_t\bar{g}_t(x_t-x^*) + \gamma_t^2\bar{g}_t^2.
\end{align}
The convergence of $\mathbb{E}[R_t]$ has been show in many previous works\footnote{Note we abuse $\mathbb{E}[\cdot]$ to represent $\mathbb{E}[\cdot |\mathcal{P}_t] = \int (\cdot)\mathcal{P}_{\epsilon_t} d\epsilon_t$ in the single-moment case, e.g. $\mathbb{E}[R_t]$ here; and to represent a time-ordered integral $\int\mathcal{P}_{\epsilon_t} d\epsilon_t\ldots\int\mathcal{P}_{\epsilon_2} d\epsilon_2\int (\cdot) \mathcal{P}_{\epsilon_1} d\epsilon_1$ in the accumulation case, e.g. $\mathbb{E}[S_T]$ below.}, here we put a simple demonstration:
\begin{theorem}\label{S_thm_mean_conv}
Assume that $\mathbb{E}[|\bar{g}_t|^2] \leq G^2$,  and that a $L$-\textit{Lipschitz} function $f(x)$ is strongly convex such that there exists a positive number $l>0$ satisfying:
\begin{align}
f(x) -  f(x') \geq \nabla_x f(x')\cdot(x-x') + \frac{l}{2}|x-x'|^2.
\end{align}
Then if we choose that $\gamma_t = \frac{1}{tl}$, we have for the Lyapunov criteria $R_t = |x_t - x^*|^2$:
\begin{align}
\mathbb{E}[R_t] \leq \frac{\max\bigg\{R_1, G^2/l^2\bigg\}}{t},
\end{align}
where $x^*$ is the optimal solution to minimize the loss function $f(x)$.
\end{theorem}
\begin{proof}
From the strong convexity, we know:
\begin{align}
f(x^*) - f(x_t) &\geq \nabla_x f(x_t)\cdot (x^*-x_t) + \frac{l}{2}|x_t - x^*|^2, \nonumber \\
f(x_t) - f(x^*) &\geq \nabla_x f(x^*)\cdot (x_t-x^*) + \frac{l}{2}|x_t - x^*|^2.
\end{align}
Combine the two inequalities above, and use that fact that $\nabla_x f(x^*) = 0$, we obtain that:
\begin{align}\label{S_ineq_1}
\nabla_x f(x_t)\cdot (x_t - x^*) \geq l\cdot |x_t - x^*|^2.
\end{align}
Besides, according to the updating rule, we have:
\begin{align}
\mathbb{E}[R_{t+1}] &= \mathbb{E}[R_t] - 2\gamma_t\mathbb{E}[\bar{g}_t\cdot(x_t - x^*)] + \gamma_t^2\mathbb{E}[|\bar{g}_t|^2] \nonumber \\
&\leq R_t - 2\gamma_t\nabla f(x_t)\cdot(x_t - x^*) + \gamma_t^2G^2.
\end{align}
Applying the above result in \eqref{S_ineq_1}, we have:
\begin{align}
\mathbb{E}[R_{t+1}] &\leq (1- 2l\gamma_t)\mathbb{E}[R_t]  + \gamma_t^2G^2.
\end{align}
Define $R\triangleq \max\big\{R_1, G^2/l^2\big\}$ for convenience, and apparently for $t=1$ we have:
\begin{align}
\mathbb{E}[R_t] \leq \frac{R}{t}.
\end{align}
Now suppose, for certain value of $t\geq 1$,  we had $\mathbb{E}[R_t] \leq \frac{R}{t}$, then according to the choice $\gamma_t = \frac{1}{tl}$ we know:
\begin{align}
\mathbb{E}[R_{t+1}] &\leq (1- \frac{2}{t})\frac{R}{t} + \frac{1}{t^2l^2} G^2 \nonumber \\
&\leq \frac{1}{t}R- \frac{2}{t^2}R + \frac{1}{t^2} \frac{G^2}{l^2} \nonumber \\
&\leq \frac{1}{t}R- \frac{1}{t^2}R  \nonumber \\
&\leq \frac{R}{t+1}.
\end{align}
This completes the proof.
\end{proof}
Now that we demonstrated the convergence of $\mathbb{E}[R_t]$ under proper conditions, the algorithm convergence would be ensured as long as the variance $Var(R_t)$ converges at the same time, which we would prove below. For convenience of analysis, we use the following mean-zero random variable:
\begin{align}
\epsilon_t \triangleq \bar{g}_t - g_t,
\end{align}
where $g_t$ represents the true-gradient at $x_t$, i.e. $g_t = \nabla_xf(x_t)$.
The statistic moments of variable $\epsilon_t$ upto the fourth order are termed as:
\begin{align}
\mathbb{E}[\epsilon_t^2] = \sigma_t^2, \qquad
\mathbb{E}[\epsilon_t^3] = \iota_t^3, \qquad
\mathbb{E}[\epsilon_t^4] = \eta_t^4.
\end{align}
Now, in a SGD algorithm, the iterating relation \eqref{S_ite_R} can be written as:
\begin{align}
R_{t+1}=&\big[R_{t} - 2\gamma_tg_t(x_t-x^*) + \gamma_t^2g_t^2\big] +\big[\gamma_t^2\epsilon_t^2 - 2\gamma_t\epsilon_td_t\big], \nonumber \\
R_{t+1}^2 =&\big[R_{t} - 2\gamma_tg_t(x_t-x^*) + \gamma_t^2g_t^2\big]^2 + \big[\gamma_t^4\epsilon_t^4 - 4 \gamma_t^3\epsilon_t^3d_t+ 4\gamma^2_t\epsilon^2_td^2_t\big]\nonumber \\
& + 2\big[\gamma_t^2\epsilon_t^2 - 2\gamma_t\epsilon_td_t\big]\cdot\big[R_{t} - 2\gamma_tg_t(x_t-x^*) + \gamma_t^2g_t^2\big].
\end{align}
Now take expectation value on both sides:
\begin{align}
\mathbb{E}[R_{t+1}]=&\big[R_{t} - 2\gamma_tg_t(x_t-x^*) + \gamma_t^2g_t^2\big] +\gamma_t^2\sigma_t^2, \nonumber \\
\mathbb{E}[R_{t+1}^2] =&\big[R_{t} - 2\gamma_tg_t(x_t-x^*) + \gamma_t^2g_t^2\big]^2 + \big[\gamma_t^4\eta_t^4 - 4 \gamma_t^3\iota_t^3d_t+ 4\gamma^2_t\sigma^2_td^2_t\big]\nonumber \\
& + 2\gamma_t^2\sigma_t^2\cdot\big[R_{t} - 2\gamma_tg_t(x_t-x^*) + \gamma_t^2g_t^2\big] \nonumber \\
=&\mathbb{E}[R_{t+1}]^2 + \big[\gamma_t^4(\eta_t^4-\sigma^4_t) - 4 \gamma_t^3\iota_t^3d_t+ 4\gamma^2_t\sigma^2_td^2_t\big],
\end{align}
Therefore the variance of the Lyapunov criteria $R_t$ at moment $t$ satisfies:
\begin{align}\label{S_var_R1}
Var(R_{t+1}) = \gamma_t^4(\eta_t^4-\sigma^4_t)+ 4\gamma^2_t\sigma^2_td^2_t  - 4 \gamma_t^3\iota_t^3d_t,
\end{align}
where we defined: 
\begin{align}
d_t \triangleq \mathbb{E}\big[x_{t+1}-x^*\big|\mathcal{P}_t\big] \equiv  x_t-\gamma_tg_t -x^*,
\end{align}
which is a constant term given the history $\mathcal{P}_t$. And we have the following result:
\begin{corollary}\label{S_thm_var_conv}
Assume a $L$-\textit{Lipschitz} $f(x)$ is strongly convex, and assume that all the first four moments of random variable $\epsilon_t$ have finite values at any moment $t$. Assume the time-dependent learning rate value $\gamma_t$ satisfies $\gamma_t\leq\frac{1}{L}$ for arbitrary time $t$, and also that $\gamma_t$ itself converges.

If the choice of $\gamma_t$ ensures the convergence of $\mathbb{E}[R_t]$, then $Var(R_{t})$ converges with a rate faster than the convergence rate of $\gamma_t^2$. 
\end{corollary}
\begin{proof}
From $L$-\textit{Lipschitz} continuous we have the following relation:
\begin{align}
|g_t| = |\nabla f(x_t) - \nabla f(x^*)| \leq L|x_t - x^*|.
\end{align}
Hence we know the constant term satisfy:
\begin{align}
|d_t| \leq |x_t - x^*| + \gamma_t|g_t| < (1 + \gamma_tL)\cdot |x_t - x^*| < 2|x_t - x^*| = 2\sqrt{\mathbb{E}[R_t]}.
\end{align}
If the convergence of $\mathbb{E}[R_t]$ would also ensure the convergence of $d_t$.

Now we write the asymptotic behavior of converging terms $\mathbb{E}[R_t]$ and $\gamma_t$ as a function of time, i.e.
\begin{align}
\exists \alpha>0, \qquad\;\: \lim_{t\rightarrow \infty} \gamma_t &\propto \lim_{t\rightarrow \infty} t^{-\alpha} = 0, \nonumber \\
\exists \alpha'>0, \quad \lim_{t\rightarrow \infty} \mathbb{E}[R_t] &\propto \lim_{t\rightarrow \infty} t^{-\alpha'} = 0.
\end{align}
Therefore we find that the first term in \eqref{S_var_R1} converges with a rate $t^{-4\alpha}$, as both $\eta_t^4$ and $\sigma_t^4$ are finite values;
the last two terms in \eqref{S_var_R1} also converge with a rate $t^{-2\alpha - \alpha'}$ and $t^{-3\alpha-\alpha'/2}$, respectively, because the $|d^t|$ and $|d_t|^2$ converge along with $\mathbb{E}[R_t]$.
The converging rate would then be:
\begin{align}
\lim_{t\rightarrow \infty} Var(R_t) \propto t^{-\tilde{\alpha}} = 0, \qquad \tilde{\alpha} = \min\bigg\{2\alpha+\alpha', 3\alpha+\frac{\alpha'}{2}, 4\alpha\bigg\} > 2\alpha.
\end{align}
This completes the proof.
\end{proof}

\subsection{Upper-bound of averaged-loss}\label{S1.2}
Here we calculate the upper bound of the following averaged-loss term, which is used as the criteria in the analysis of stochasticity error:
\begin{align}\label{S_ave_loss}
\frac{1}{T}\sum_{t=1}^T\big[f(x_{t}) - f^*\big].
\end{align}
We used the following result in the main text, which we will prove below:
\begin{lemma}\label{S_lem_los_upbnd}
Assume the $L$-Lipschitz continuity for the convex loss function $f(x)$, the learning rate $\gamma_t \leq \frac{1}{L}$ at any time $t$, and a bounded region for the parameter $|x_t-x^*|^2 \leq M^2$. Then the criteria $\frac{1}{T}\sum_{t=1}^{T}\big[f(x_{t}) - f^*\big]$ would be suppressed below the following upper bound:
\begin{align}\label{S_upbnd}
S_T &= \frac{1}{\gamma_TT}\cdot\frac{M^2}{2}+ \frac{1}{2T}\sum_{t=1}^{T}\gamma_t\big(1+L\gamma_t\big)\epsilon_t^2
- \frac{1}{T}\sum_{t=1}^{T} \epsilon_t\big(x_t-L\gamma_t^2g_t-x^*\big).
\end{align}
\end{lemma}
\begin{proof}
From the Lipschitz continuity , we know:
\begin{align}
f(x_{t+1}) - f(x_t) &\leq \nabla f(x_t)\cdot (x_{t+1} - x_t) + \frac{1}{2}L|x_{t+1} - x_t|^2 \nonumber \\
f_{t+1} &\leq f_t - \gamma_tg_t\bar{g}_t + \frac{1}{2}L\gamma_t^2 \bar{g}_t^2
\end{align}
Further we use the convexity property: $f(x_t) - f(x^*) \leq \nabla f(x_t)\cdot (x_t - x^*)$, and obtain that:
\begin{align}
f_{t+1} &\leq f^* + g_t(x_t-x^*) - \gamma_tg_t\bar{g}_t + \frac{1}{2}L\gamma_t^2 \bar{g}_t^2.
\end{align}
Noticing $\gamma_t \leq \frac{1}{L}$, we could reorganize it as:
\begin{align}
f_{t+1} - f^* \leq& g_t(x_t-x^*) - \gamma_t\bigg[1-\frac{1}{2}L\gamma_t\bigg]g_t^2 \nonumber \\
&-\gamma_tg_t\epsilon_t + L\gamma_t^2g_t\epsilon_t + \frac{1}{2}L\gamma_t^2 \epsilon_t^2\nonumber \\
\leq& g_t(x_t-x^*) - \frac{1}{2}\gamma_t(g_t^2+\epsilon_t^2) + \frac{\gamma_t}{2}\cdot\big[1+L\gamma_t\big]\epsilon_t^2 \nonumber \\
&+  L\gamma_t^2g_t\epsilon_t.
\end{align}
Complete the square, we have:
\begin{align}
f_{t+1} - f^* \leq& \frac{1}{2\gamma_t}\bigg[|x_t - x^*|^2 - |x_{t+1} - x^*|^2\bigg] + \frac{\gamma_t}{2}\cdot\big[1+L\gamma_t\big]\epsilon_t^2 \nonumber \\
& -  \epsilon_t\big[x_t-L\gamma_t^2g_t-x^*\big].
\end{align}
Taking summation on both sides from $t=0$ to $t=T-1$, we obtain:
\begin{align}
\frac{1}{T}\sum_{t=0}^{T-1}\bigg[f(x_{t+1}) - f^*\bigg] \leq& \frac{1}{T} \bigg[
\frac{1}{2\gamma_0}|x_0-x^*|^2 - \frac{1}{2\gamma_T}|x_0-x^*|^2 + \frac{1}{2}\sum_{t=0}^{T-1}\bigg(\frac{1}{\gamma_{t+1}} - \frac{1}{\gamma_{t}}\bigg)\cdot |x_{t+1}-x^*|^2 \nonumber \\
&+ \sum_{t=0}^{T-1}\frac{\gamma_t}{2}\cdot\big[1+L\gamma_t\big]\epsilon_t^2- \sum_{t=0}^{T-1} \epsilon_t\big[x_t-L\gamma_t^2g_t-x^*\big]\bigg] .
\end{align}
Therefore we obtain the following upper bound for the concerned criteria:
\begin{align}
\frac{1}{T}\sum_{t=1}^{T}\bigg[f(x_{t}) - f^*\bigg] 
\leq \frac{1}{T} \bigg[&
\frac{1}{2\gamma_0}|x_0-x^*|^2 + \frac{1}{2}\sum_{t=1}^{T}\bigg(\frac{1}{\gamma_{t}} - \frac{1}{\gamma_{t-1}}\bigg)\cdot |x_t-x^*|^2 \nonumber \\
&+ \sum_{t=1}^{T-1}\frac{\gamma_t}{2}\cdot\big[1+L\gamma_t\big]\epsilon_t^2-\sum_{t=1}^{T-1} \epsilon_t\big[x_t-L\gamma_t^2g_t-x^*\big]\bigg] \nonumber \\
 \leq \frac{1}{T} \bigg[&
\frac{M^2}{2\gamma_0} + \frac{M^2}{2}\sum_{t=1}^{T}\bigg(\frac{1}{\gamma_{t}} - \frac{1}{\gamma_{t-1}}\bigg) + \sum_{t=1}^{T-1}\frac{\gamma_t}{2}\cdot\big[1+L\gamma_t\big]\epsilon_t^2-\sum_{t=1}^{T-1} \epsilon_t\big[x_t-L\gamma_t^2g_t-x^*\big]\bigg] \nonumber \\
= \frac{1}{T} \bigg[&
\frac{1}{\gamma_T}\cdot\frac{M^2}{2} + \frac{1}{2}\sum_{t=1}^{T-1}\gamma_t\big(1+L\gamma_t\big)\epsilon_t^2-\sum_{t=1}^{T-1} \big(x_t-L\gamma_t^2g_t-x^*\big)\epsilon_t\bigg].
\end{align}
This completes the proof.
\end{proof}

With the derived upper-bound, we will rigorously prove the convergence of $S_T$ in the next section, which includes both the convergence of expectation value and the variance term. Since the analysis is mainly aimed at showing the impact of the stochastic error, it is reasonable to focus on the upper-bound other than the criteria of averaged-loss in \eqref{S_ave_loss}, which is difficult to calculate.

\subsection{Convergence of the upper-bound expectation value}\label{S1.3}

\begin{theorem}\label{S_them1_conv_lr}
Assume the $L$-Lipschitz continuity for the convex loss function $f(x)$, the learning rate $\gamma_t \leq \frac{1}{L}$ at any time $t$, and a bounded region for the parameter $|x_t-x^*|^2 \leq M^2$. Further assume a finite upper bound exists such that $\sigma_0^2\geq \sigma_t^2$ for any time $t$. Then the expectation of upper bound $S_T$ \eqref{S_upbnd} satisfies:
\begin{align}\label{S_S2}
\mathbb{E}[S_T] \leq \frac{1}{\gamma_TT}\cdot\frac{M^2}{2} + \frac{\sum_t \gamma_t}{T}\cdot \sigma_0^2.
\end{align}
\end{theorem}
\begin{proof}
Make use of the fact that $|x_t- x^*|^2\leq M^2$, we obtain:
\begin{align}
S_T < 
\frac{1}{\gamma_TT}\cdot\frac{M^2}{2} + \frac{1}{2T}\sum_{t=1}^{T-1}\gamma_t\big[1+L\gamma_t\big]\epsilon_t^2
- \frac{1}{T}\sum_{t=1}^{T-1} \epsilon_t\big[x_t-L\gamma_t^2g_t-x^*\big],
\end{align}
where the last term vanishes. Take expectation on both sides:
\begin{align}
\mathbb{E}[S_T] &= \frac{M^2}{2\gamma_TT} + \frac{1}{2T}\sum_{t=1}^{T}\gamma_t\big[1+L\gamma_t\big]\sigma_t^2 .
\end{align}
Since there is an upper bound $\sigma_0^2$ for all $\sigma_t^2$, we have for expectation value:
\begin{align}
\mathbb{E}[S_T] &\leq \frac{1}{\gamma_TT}\cdot \frac{M^2}{2} + \frac{\sum_{t=1}^{T}\gamma_t}{T} \cdot \sigma_0^2.
\end{align}
This completes the proof.
\end{proof}

Now suppose we design a vanishing learning rate:
\begin{align}
\gamma_t \sim t^{-\alpha}.
\end{align}
The first term of \eqref{S_S2} would behave as $\mathcal{O}(T^{\alpha -1})$, which suggests a value range $0< \alpha < 1$. The second term can be approximated as an integral by using the Euler-Maclaurin formula:
\begin{align}
\frac{1}{T} \sum_{t=1}^{T}\gamma_t &\approx  \frac{1}{T} \int_{1}^{T} \gamma_tdt + \frac{1+T^{-\alpha}}{2T} + \frac{\alpha}{6T}\cdot\frac{1 - T^{-\alpha-1}}{2!} + \frac{\alpha(\alpha+1)(\alpha+2)}{30T}\cdot\frac{1 - T^{-\alpha-3}}{4!} \nonumber \\
&\approx \frac{1}{1-\alpha}\cdot \frac{1}{T^{\alpha}} + \bigg[\frac{1}{2} - \frac{1}{1-\alpha} + \frac{\alpha}{12} + \frac{\alpha(\alpha+1)(\alpha+2)}{720}\bigg]\cdot \frac{1}{T} + \mathcal{O}(T^{-\alpha-1}).
\end{align}
As we mentioned above, $\alpha < 1$, thus the second term approximately behaves as $\mathcal{O}(T^{-\alpha})$. Therefore, the expectation value \eqref{S_S2} behaves as:
\begin{align}\label{S_trade-off}
\mathbb{E}[S_T] &\sim \mathcal{O}(T^{\alpha -1}) + \mathcal{O}(T^{-\alpha}),
\end{align}
which is used in the main text.

\subsection{Convergence of the upper-bound variance}\label{S1.4}
Define the average values of second the forth moments of $\epsilon_t$ respectively as: 
\begin{align}
\bar{\sigma}^2 \triangleq \frac{1}{T}\sum_t \sigma_t^2, \qquad \bar{\eta}^4 \triangleq \frac{1}{T}\sum_t \eta_t^4,
\end{align}
both of which are assumed to be positive values, then we have the following result:
\begin{theorem}\label{S_them2_conv_lr}
Assume the $L$-Lipschitz continuity for the convex loss function $f(x)$, the learning rate $\gamma_t \leq \frac{1}{L}$ at any moment, and a bounded region for the parameter $|x_t-x^*|^2 \leq M^2$; and for simplicity, assume the distribution of $\epsilon_t$ at any moment $t$ to be symmetric, then the following result holds:
\begin{align}
Var(S_T) \leq \frac{1}{T}\bigg[4M^2\bar{\sigma}^2 + \frac{\bar{\eta^4}}{L^2}\bigg].
\end{align}
\end{theorem}
\begin{proof}
Define the following constant term:
\begin{align}
A = \frac{1}{2\gamma_0}|x_0-x^*|^2 + \frac{1}{2}\sum_{t=1}^{T}\bigg(\frac{1}{\gamma_{t}} - \frac{1}{\gamma_{t-1}}\bigg)\cdot |x_t-x^*|^2.
\end{align}
Therefore the upper-bound can be expressed as:
\begin{align}
S_T = \frac{1}{T} \bigg[
A + \frac{1}{2}\sum_{t=1}^{T-1}\gamma_t\big(1+L\gamma_t\big)\epsilon_t^2-\sum_{t=1}^{T-1} \epsilon_t\big(x_t-L\gamma_t^2g_t-x^*\big)\bigg]
\end{align}
Firstly we have the following expectation values:
\begin{align}
\mathbb{E}[S_T]^2 = \frac{1}{T^2}&\bigg[A^2+  \frac{1}{4}\bigg[\sum_{t=1}^{T-1}\gamma_t\big(1+L\gamma_t\big)\sigma_t^2\bigg]^2+  A\sum_{t=1}^{T-1}\gamma_t\big(1+L\gamma_t\big)\sigma^2 \bigg],\nonumber \\
\mathbb{E}[S_T^2] = \frac{1}{T^2}&\bigg[A^2 +  \frac{1}{4}\bigg[\sum_{t=1}^{T-1}\gamma_t\big(1+L\gamma_t\big)\sigma_t^2\bigg]^2+ \sum_{t=1}^{T-1} \sigma^2_t\big(x_t-L\gamma_t^2g_t-x^*\big)^2 \nonumber \\
&+ A\sum_{t=1}^{T-1}\gamma_t\big(1+L\gamma_t\big)\sigma^2
 - \sum_{t=1}^{T-1}\gamma_t\big(1+L\gamma_t\big)\cdot\big(x_t-L\gamma_t^2g_t-x^*\big)\cdot\iota_t^3 \nonumber \\
&+  \frac{1}{4}\sum_{t=1}^{T-1}\gamma_t^2\big(1+L\gamma_t\big)^2\cdot(\eta_t^4 - \sigma_t^4)
\bigg].
\end{align}
Under the assumption that the distribution of $\epsilon_t$ is symmetric, we know $\iota_t \equiv 0$, therefore the variance term is:
\begin{align}
Var(S_T) &= \frac{1}{T^2}\bigg[\sum_{t=1}^{T-1} \sigma^2_t\big(x_t-L\gamma_t^2g_t-x^*\big)^2
+ \frac{1}{4}\sum_{t=1}^{T-1}\gamma_t^2\big(1+L\gamma_t\big)^2\cdot(\eta_t^4 - \sigma_t^4)
\bigg].
\end{align}
From the Lipschitz continuity and the assumptions that $|x_t-x^*|^2 \leq M^2$ and $\gamma_t \leq \frac{1}{L}$:
\begin{align}
Var(S_T) &\leq \frac{1}{T^2}\bigg[\sum_{t=1}^{T-1} \sigma^2_t\bigg(|x_t-x^*| + L\gamma_t^2|g_t|\bigg)^2
+ \frac{1}{4}\sum_{t=1}^{T-1}\gamma_t^2\big(1+L\gamma_t\big)^2\cdot(\eta_t^4 - \sigma_t^4)
\bigg] \nonumber \\
&\leq \frac{1}{T^2}\bigg[\sum_{t=1}^{T-1} 4\sigma^2_t|x_t-x^*|^2
+ \sum_{t=1}^{T-1}\gamma_t^2(\eta_t^4 - \sigma_t^4)
\bigg] \nonumber \\
&\leq \frac{4M^2\bar{\sigma}^2}{T} + \frac{1}{T^2}\sum_{t=1}^{T-1}\gamma_t^2\eta_t^4 \nonumber \\
&\leq \frac{1}{T}\bigg[4M^2\bar{\sigma}^2 + \frac{\bar{\eta^4}}{L^2}\bigg]
\end{align}
This completes the proof.
\end{proof}

\subsection{An approximating solution of minimization problem}\label{S1.5}
In the upper-bound $S_T$ in \eqref{S_upbnd}, the last term averages to zero after taking the expectation. Therefore the convergence is mainly affected by the following two terms:
\begin{align}
\mathbb{E}[S_T] = \frac{1}{\gamma_TT}\cdot\frac{M^2}{2}+ \frac{1}{2T}\sum_{t=1}^{T}\gamma_t\big(1+L\gamma_t\big)\sigma^2.
\end{align}
And the analysis obtained in \eqref{S_trade-off} suggests the second term pollutes the convergence of the first one, and results into a trade-off. Therefore we are aiming at minimize the following term by some regularization design:
\begin{align}\label{S_p0}
\sum_{t=1}^{T}\gamma_t\big(1+L\gamma_t\big)\sigma^2.
\end{align}
As explained in the main test, a well-defined optimization problem can be stated as:
\begin{equation}
\left\{\arraycolsep=1.4pt\def\arraystretch{2.2}
\begin{array}{l}
\lambda_t = \underset{\{\lambda_t\}}{\min}\bigg\{\sum_t\lambda_t\big[1+\frac{\lambda_t}{\lambda_0}\big]\sigma_t^2  \:\bigg|\: \{\sigma_t^2\},L \bigg\}, \\
\sum_t \lambda_t =T, \qquad \lambda_t \in [ \lambda_{min}, \lambda_{max}].
\end{array}
\right.
\end{equation}
\begin{theorem}\label{S_them_solution}
Assume the $L$-Lipschitz continuity for the convex loss function $f(x)$, and that $\alpha < \frac{1}{L}$; further assume that  local variances $\sigma_t^2$ at different time $t$ satisfy the following condition:
 \begin{align}\label{S_assumption}
\max_{s,t}\big(|\sigma_t^2 - \sigma_s^2|\big) \ll \bar{\sigma}^2.
\end{align}
Then the stochastic error accumulation \eqref{S_p0} could be approximately minimized by the following regularizer:
\begin{align}
\lambda_t = \bigg(1 + \frac{\lambda_0}{2}\bigg)\frac{\bar{\sigma}^2}{\sigma_t^2} - \frac{\lambda_0}{2},
\end{align}
when the convergence of $\mathbb{E}[S_T]$ is still ensured.
\end{theorem}
\begin{proof}
In the case where $max_{s,t}\big(|\sigma_t^2 - \sigma_s^2|\big) \ll \bar{\sigma}^2$, the second constraint $\lambda_t \in [ \lambda_{min}, \lambda_{max}]$ would be directly satisfied after solving the following approximating problem:
\begin{equation}
\left\{\arraycolsep=1.4pt\def\arraystretch{2.2}
\begin{array}{l}
\lambda_t = \underset{\{\lambda_t\}}{\min}\bigg\{\sum_t\lambda_t\big[1+\frac{\lambda_t}{\lambda_0}\big]\sigma_t^2  \:\bigg|\: \{\sigma_t^2\},L \bigg\}, \\
\sum_t \lambda_t =T.
\end{array}
\right.
\end{equation}
This problem can be solved by introducing a Lagrangian multiplier $\xi$, and formulate the optimization as:
\begin{align}
\lambda_t &= \underset{\{\lambda_t\}}{\min}\Bigg(\sum_{t=1}^T\lambda_t\bigg(1+\frac{\lambda_t}{\lambda_0}\bigg)\sigma_t^2  
- \xi \bigg[\sum_{t=1}^T \lambda_t - T\bigg]\Bigg).
\end{align}
For convenience, we define a new variable:
\begin{align}
q_t = \frac{\lambda_t}{\lambda_0} \leq 1,
\end{align}
then we can in turn focus on minimizing the following target function:
\begin{align}
\tilde{Q}_T &= \lambda_0\Bigg[\sum_{t=1}^T\sigma_t^2  q_t\big(q_t+1\big)
- \sum_{t=1}^T \xi q_t  + \frac{T\xi }{\lambda_0}\Bigg] \nonumber \\
&= \lambda_0\Bigg[\sum_{t=1}^T\sigma_t^2  \bigg(q_t+\frac{1}{2}\bigg)^2
- \sum_{t=1}^T \xi q_t  \Bigg] + \sum_{t=1}^T\bigg(\xi  - \frac{1}{4}\lambda_0\sigma_t^2\bigg).
\end{align}
The first order condition simply gives:
\begin{align}
2\sigma_t^2  \bigg(q_t+\frac{1}{2}\bigg) - \xi &= 0 \qquad \Longrightarrow\qquad q_t = \frac{1}{2}\bigg(\frac{\xi}{\sigma_t^2} - 1\bigg).
\end{align}
Substitute this into the average global constraint, we obtain:
\begin{align}
\sum_{t=1}^T \frac{\lambda_0}{2}\bigg(\frac{\xi}{\sigma_t^2} - 1\bigg) &= T \qquad \Longrightarrow\qquad \xi = \tilde{\sigma}^2\bigg(\frac{2}{\lambda_0} + 1\bigg),
\end{align}
where we define the inverse-averaged variance:
\begin{align}\label{S_invers_ave}
\tilde{\sigma}^2 =  \Bigg(\frac{1}{T}\sum_{t=1}^T \frac{1}{\sigma_t^2}\Bigg)^{-1}.
\end{align}
Then the solution without considering the bounded region constraint would be:
\begin{align}
\lambda_t = \bigg(1 + \frac{\lambda_0}{2}\bigg)\frac{\tilde{\sigma}^2}{\sigma_t^2} - \frac{\lambda_0}{2}.
\end{align}
Again consider the condition $max_{s,t}\big(|\sigma_t^2 - \sigma_s^2|\big) \ll \bar{\sigma}^2$, we can write $\sigma_t^2 = \bar{\sigma}^2 + \theta_t$, with $|\theta_t|\ll \bar{\sigma}$, therefore:
\begin{align}
\tilde{\sigma}^2 &=  \Bigg[\frac{1}{T}\sum_{t=1}^T \frac{1}{\bar{\sigma}^2 + \theta_t}\Bigg]^{-1} \simeq  \bar{\sigma}^2\Bigg[1 + \frac{1}{T}\sum_{t=1}^T \frac{\theta_t}{\bar{\sigma}^2}\Bigg].
\end{align}
Thus to the lowest order approximation, we have $\tilde{\sigma}^2\approx \bar{\sigma}^2$, and an approximating solution would be:
\begin{align}\label{S_sl1}
\lambda_t = \bigg(1 + \frac{\lambda_0}{2}\bigg)\frac{\bar{\sigma}^2}{\sigma_t^2} - \frac{\lambda_0}{2}.
\end{align}
Besides, since we assumed $max_{s,t}\big(|\sigma_t^2 - \sigma_s^2|\big) \ll \bar{\sigma}^2$, which indicates $|\lambda_t-1|\ll 1$, with a proper choice of $\alpha<\frac{1}{L}$, we always have $\gamma_t \leq\frac{1}{L}$. Therefore all assumptions mentioned in \textbf{Theorem \ref{S_them1_conv_lr}} are satisfied, and the convergence of $\mathbb{E}[S_T]$ is then ensured.

This completes the proof.
\end{proof}

\newpage
\section{Analysis and comments for the main text}
In this section, we would discuss several important details related to the implementation of the proposed regularization.

\subsection{Applicable bounded solution form}\label{S2.1}
In the main text, we mentioned a practically applicable form of the approximating solution would be:
\begin{align}\label{S_sl2}
\lambda_t  = \frac{1 + s}{1 + s \cdot \frac{\sigma_t^2}{\bar{\sigma}^2}}.
\end{align}
Now let us explain it as a bounded-approximation of \eqref{S_sl1}. 

The problem of \eqref{S_sl1}, as we mentioned in the main text and above, is that it is only valid under the following assumption:
\begin{align}
\max_{s,t}\big(|\sigma_t^2 - \sigma_s^2|\big) \ll \bar{\sigma}^2,
\end{align}
in which case we expect the bounded region constraint $\lambda_t \in [\lambda_{min}, \lambda_{max}]$ could be satisfied automatically. While in general cases, when $\sigma_t^2$ is too small, $\lambda_t$ would explode to large values, and when $\sigma_t^2$ is too large, $\lambda_t$ would be even negative. Therefore, we would like to regularize the solution \eqref{S_sl1} to bound it into a proper range.

There are different ways to bound a generic real number, and we choose the sigmoid function. More specifically, we consider the following transformation:
\begin{align}\label{S_sl3}
\bigg(1 + \frac{\lambda_0}{2}\bigg)\frac{\bar{\sigma}^2}{\sigma_t^2} - \frac{\lambda_0}{2}\quad \longrightarrow\quad
\frac{a}{1+\exp{\bigg[\frac{\lambda_0}{2} - \bigg(1+\frac{\lambda_0}{2}\bigg)\frac{\bar{\sigma}^2}{\sigma_t^2} \bigg]}}.
\end{align}
Note in the case where the heteroskedasticity is absent, i.e. $\sigma_t^2 \equiv \bar{\sigma}^2$, if we choose $a = (1+\frac{1}{e})$, the above expression suggests $\lambda_t = 1$, and our method goes back to the conventional one.
Now, although \eqref{S_sl3} can already be used the in calculation, in practice, to further simplify the expression, we would propose another approximating form. 

Reminded that the solution \eqref{S_sl1} itself is obtained under the condition \eqref{S_assumption}, we know the above bounded form \eqref{S_sl3} can also represent the true solution if \eqref{S_assumption} is true. In that case, again by writing $\sigma_t^2 = \bar{\sigma}^2 + \theta_t$, we can expand \eqref{S_sl3} as:
\begin{align}
\lambda_t \sim& \frac{a}{1+\exp{\bigg[\frac{\lambda_0}{2} - \bigg(1+\frac{\lambda_0}{2}\bigg)\frac{\bar{\sigma}^2}{\sigma_t^2} \bigg]}} \nonumber \\
\approx& \frac{a}{1+\exp{\bigg[\frac{\lambda_0}{2} - \bigg(1+\frac{\lambda_0}{2}\bigg)\bigg(1 - \frac{\theta_t}{\bar{\sigma}^2}\bigg)\bigg]}} \nonumber \\
\approx& \frac{a'}{e+\bigg[1 + \big(1+\frac{\lambda_0}{2}\big)\bigg(\frac{\sigma^2_t}{\bar{\sigma}^2}-1\bigg)\bigg]} \nonumber \\
=& \frac{1+ s}{1+ s\cdot \frac{\sigma^2_t}{\bar{\sigma}^2}}.
\end{align}
where $s = \bigg(\frac{1+\lambda_0/2}{e- \lambda_0/2}\bigg)$, and we set the numerator value to reach $\lambda_t=1$ in constant $\sigma_t^2$ case. This is exactly the bounded form we implemented in \eqref{S_sl2}. Besides, the above expression also suggests a regular range for the choice of $x$ in practice: since $\lambda_0\sim\frac{1}{\alpha_0L}$, we expect it is of order $\mathcal{O}(1)$ and can estimate the value with the range $(1,5)$; therefore we can estimate the value $x$ with the range $(0.5, 20)$.

\subsection{From $v^2_t$ to $\sigma_t^2$}\label{S2.2}
We assume the \textit{i.i.d.} nature for all data samples. In a mini-batch version algorithm, there are actually two sets of random variables and corresponding statistics: $\{\delta_{i,t}\}$ and $\{\bar{\delta}_t\}$. By definition:
\begin{align}
\bar{\delta}_t  = \frac{1}{m}\sum_{i=1}^m \delta_{i,t},
\end{align}
which means the variable $\bar{\delta}_t$ is the summation of $m$ iid variables. Then from the law of large number, we know for the first order moment:
\begin{align}
\mathbb{E}[\delta_{i,t}] = \mathbb{E}[\bar{\delta}_t] = \delta_t.
\end{align}
And from CLT, we also have:
\begin{align}
\sigma_0^2(x_t) \triangleq\mathbb{E}[(\delta_{i,t}-\delta_t)^2], \qquad \sigma_t^2 \triangleq \mathbb{E}[(\bar{\delta}_t-\delta_t)^2] = \frac{\sigma_0^2(x_t)}{m}.
\end{align}
Now, in addition, according to Cochran's theorem, we actually know that the following two second-order variables, which are also random variables before taken the expectation values, should obey distributions:
\begin{align}
(\delta_{i,t}-\bar{\delta}_t)^2 &\sim \frac{\sigma_0^2(x_t)}{m}\chi^2_{m-1}, \qquad (\bar{\delta}_t-\delta_t)^2 \sim \frac{\sigma_0^2(x_t)}{m}\chi^2_{1}
\end{align}
The first quantity represents the deviation of a single sample in the mini-batch from the sample-mean, and the second quantity represents the deviation of sample-mean from real mean.

As we mentioned in the main text, in practice we can only use the term $(\delta_{i,t}-\bar{\delta}_t)^2$, whose expectation value is:
\begin{align}
v_t^2 \triangleq \mathbb{E}[(\delta_{i,t}-\bar{\delta}_t)^2] = \frac{m-1}{m}\sigma_0^2(x_t) = (m-1)\sigma_t^2.
\end{align}
Therefore, the ratio with respect to the normal-averaged-variance $\bar{v}^2 \triangleq \frac{1}{T}\sum_tv_t^2$ would satisfy:
\begin{align}
\frac{v_t^2}{\bar{v}^2} = \frac{(m-1)\sigma_t^2}{(m-1)\bar{\sigma}^2}= \frac{\sigma_t^2}{\bar{\sigma}^2},
\end{align}
which makes the ratio $\frac{v_t^2}{\bar{v}^2}$ practically a legal choice.

\subsection{Scale-free variance}\label{S2.3}

We introduce the concept of \textit{scale-free variance} $\rho_t$ in the main text. Here we would discuss the idea and forms of $\rho_t$ for practical usage. Suppose at two different moments $t$ and $t'$, we accidentally find the following relation:
\begin{align}
\mathrm{g}_{i,t} = 2\mathrm{g}_{i,t'}, \quad \forall i\in[1,m].
\end{align}
Then we immediately have:
\begin{align}
v_t^2 = 4v^2_{t'},
\end{align}
which means the variance scales with the magnitude of gradients. If we directly use $\frac{v_t^2}{\bar{v}^2}$ to regularize the learning-rate, then we would have:
\begin{align}
\gamma_t < \gamma_{t'}.
\end{align}
However, if we review our intuition of the regularization, we notice that the "uncertainty" at both steps should be the same, which suggests us to use another variance term which is free from the scaling problem, and more truthfully represents and compares the uncertainty at different moments. The simplest solution is to normalize by the $\bar{g}_t^2$, i.e.
\begin{align}
\rho_t = \frac{v_t^2}{\bar{g}_t^2}.
\end{align}
Then, for example in the above case, we have $\rho_t =\rho_{t'}$, which is a desired result. This is the form we used in \textit{VR-SGD}.

\subsection{Scale-free variance in Adam and above}\label{S2.4}
In this section, we would derive the variance of a general increment term $\bar{\delta}_t$, which has different meaning in SGD, Heavy-Ball, Adam, AdaGrad, RMSprop, and so on. In some of these methods, the randomness enters the expression of $\bar{\delta}_t$ also through the denominator, such as Adam and RMSprop, which makes a rigorous derivation inaccessible for general cases discussion. To progress, we would make approximations and consider the lowest order expression for random variables in denominator. The approximating condition may not hold, of course, in general cases, but the derived expression would still be an applicable approximation.

Firstly, we consider a general gradient-momentum term in the following form:
\begin{align}
p_t = (1-\beta)\sum_{\tau}\beta^{t-\tau}\bar{\delta}_{\tau}.
\end{align}
We play the old trick by separate the mean-zero random variable $\epsilon_t$ at each moment, and obtain:
\begin{align}
p_t &= (1-\beta)\sum_{\tau}\beta^{t-\tau}\delta_{\tau} + (1-\beta)\sum_{\tau}\beta^{t-\tau}\epsilon_{\tau}, \nonumber \\
p^2_t &= \bigg((1-\beta)^2\sum_{\tau}\beta^{t-\tau}\delta_{\tau}\bigg)^2 + (1-\beta)^2\sum_{\tau,\tau'}\beta^{2t-\tau-\tau'}\epsilon_{\tau}\epsilon_{\tau'}+ (1-\beta)^2\sum_{\tau,\tau'}\beta^{2t-\tau-\tau'}\epsilon_{\tau}\delta_{\tau'}.
\end{align}
Taking the assumption that $\epsilon_t$ at different moments are independent to each other:
\begin{align}
\mathbb{E}[p_t] &= (1-\beta)\sum_{\tau}\beta^{t-\tau}\delta_{\tau}, \nonumber \\
\mathbb{E}[p^2_t] &= \bigg((1-\beta)^2\sum_{\tau}\beta^{t-\tau}\delta_{\tau}\bigg)^2 + (1-\beta)^2\sum_{\tau,\tau'}\beta^{2t-\tau-\tau'}\mathbb{E}[\epsilon_{\tau}\epsilon_{\tau'}] \nonumber \\
&= \bigg((1-\beta)^2\sum_{\tau}\beta^{t-\tau}\delta_{\tau}\bigg)^2 + (1-\beta)^2\sum_{\tau}(\beta^2)^{t-\tau}\sigma_{\tau}^2.
\end{align}
Therefore we obtain the variance should be:
\begin{align}
\sigma_{t}^2(p) = \frac{(1-\beta)^2}{(1-\beta^2)}\cdot \bigg[(1-\beta^2)\sum_{\tau}(\beta^2)^{t-\tau}\sigma_{\tau}^2\bigg].
\end{align}
And we can assign the following iteration formula to calculate this momentum-variance in practice:
\begin{align}
\tilde{\sigma}_{t+1}^2(p) &= \beta^2\tilde{\sigma}_{t}^2(p) + (1-\beta^2)\sigma_t^2,\nonumber \\
\sigma_{t+1}^2(p) &= \frac{(1-\beta)^2}{(1-\beta^2)}\cdot\tilde{\sigma}_{t+1}^2(p).
\end{align}

Secondly, we derive the lowest order approximation for the random variable in the denominator:
\begin{align}
\frac{1}{\bar{q}} &= \frac{1}{q + \epsilon} = \sum_{k=0}^{\infty}(-1)^k\frac{\epsilon^k}{q^{k+1}}.
\end{align}
In the case where $\frac{\epsilon}{q}\ll 1$, we can use the following approximation upto the second order:
\begin{align}
\frac{1}{\bar{q}} &\approx \frac{1}{q}\bigg( 1 - \frac{\epsilon}{q} + \frac{\epsilon^2}{q^2} \bigg), \nonumber \\
\frac{1}{\bar{q}^2} &\approx \frac{1}{q^2}\bigg( 1 - 2\frac{\epsilon}{q} + 3\frac{\epsilon^2}{q^2} \bigg).
\end{align}
Therefore we have the variance approximation as:
\begin{align}
\sigma^2\bigg(\frac{1}{\bar{q}}\bigg) &\approx \frac{\sigma^2}{q^4}.
\end{align}

With the above analysis and approximations in hand, we can now derive the scale-free variance in Adam now. Approximately, one can view the increment term in Adam as:
\begin{align}
\bar{\delta}_t \sim \frac{\bar{p}_{t|1}}{\sqrt{\bar{p}^2_{t|2}}} = \frac{\bar{p}_{t|1}}{|\bar{p}_{t|2}|},
\end{align}
where the two terms are updated in the following way:
\begin{align}
\bar{p}_{t+1|1} &= \beta^1\bar{p}_{t|1} + (1-\beta^1)\bar{g}_{t},\nonumber \\
\bar{p}^2_{t+1|2} &= \beta^2\bar{p}^2_{t|2} + (1-\beta^2)\bar{g}^2_{t}
\end{align}
We use the 2nd power notation in $\bar{p}^2_{t|2}$ to trace the order in term of $\bar{g}_t$. Note, according to the above updating rule, the rigorous way to consider the randomness would be:
\begin{align}
\bar{\delta}_t &\sim \frac{(1-\beta_1)\sum_{\tau}\beta_1^{t-\tau}(g_t + \epsilon_{\tau})}{\sqrt{(1-\beta_2)\sum_{\tau}\beta_2^{t-\tau}(g_{\tau} + \epsilon_{\tau})^2}}
\approx\frac{p_{t|1} + (1-\beta_1)\sum_{\tau}\beta_1^{t-\tau}\epsilon_{\tau}}{\sqrt{p^2_{t|2} + 2(1-\beta_2)\sum_{\tau}\beta_2^{t-\tau} g_{\tau}\epsilon_{\tau} + (1-\beta_2)\sum_{\tau}\beta_2^{t-\tau} \epsilon^2_{\tau}}}.
\end{align}
We want to calculate the variance of the above term, where for the denominator we would need a Taylor expansion as before. Since the variance is of the second order, during the calculation we would only keep terms upto the second order (after square):
\begin{align}
\frac{|\bar{p}_{t|2}|}{\bar{p}_{t|1}}\cdot\bar{\delta}_t &\approx \bigg[1 + \frac{(1-\beta_1)}{p_{t|1}}\sum_{\tau}\beta_1^{t-\tau}\epsilon_{\tau}\bigg]\cdot \bigg[1 + \frac{2(1-\beta_2)}{p^2_{t|2}}\sum_{\tau}\beta_2^{t-\tau} g_{\tau}\epsilon_{\tau}+ \frac{(1-\beta_2)}{p^2_{t|2}}\sum_{\tau}\beta_2^{t-\tau} \epsilon^2_{\tau}\bigg]^{-\frac{1}{2}}\nonumber \\
&\approx \bigg[1 + \epsilon_{t|1}\bigg]\cdot \bigg[1 - \frac{(1-\beta_2)}{p^2_{t|2}}\sum_{\tau}\beta_2^{t-\tau} g_{\tau}\epsilon_{\tau} - \frac{(1-\beta_2)}{2p^2_{t|2}}\sum_{\tau}\beta_2^{t-\tau} \epsilon^2_{\tau} + \frac{3(1-\beta_2)^2}{2p^4_{t|2}}\sum_{\tau, \tau'}\beta_2^{2t-\tau-\tau'} g_{\tau}g_{\tau'}\epsilon_{\tau}\epsilon_{\tau'}\bigg],\nonumber \\
\bigg(\frac{|\bar{p}_{t|2}|}{\bar{p}_{t|1}}\cdot\bar{\delta}_t\bigg)^2 &\approx \bigg[1 + \frac{(1-\beta_1)}{p_{t|1}}\sum_{\tau}\beta_1^{t-\tau}\epsilon_{\tau}\bigg]^2 \cdot \bigg[1 + \frac{2(1-\beta_2)}{p^2_{t|2}}\sum_{\tau}\beta_2^{t-\tau} g_{\tau}\epsilon_{\tau}\bigg]^{-1} \nonumber\\
&\approx \bigg[1 + \epsilon_{t|1}\bigg]^2 \cdot \bigg[1 - \frac{2(1-\beta_2)}{p^2_{t|2}}\sum_{\tau}\beta_2^{t-\tau} g_{\tau}\epsilon_{\tau} - \frac{(1-\beta_2)}{p^2_{t|2}}\sum_{\tau}\beta_2^{t-\tau} \epsilon^2_{\tau} + \frac{4(1-\beta_2)^2}{p^4_{t|2}}\sum_{\tau, \tau'}\beta_2^{2t-\tau-\tau'} g_{\tau}g_{\tau'}\epsilon_{\tau}\epsilon_{\tau'}\bigg],
\end{align}
where we defined for simplicity:
\begin{align}
\epsilon_{t|1} = \frac{(1-\beta_1)}{p_{t|1}}\sum_{\tau}\beta_1^{t-\tau}\epsilon_{\tau}.
\end{align}
Now we calculate the expectation value:
\begin{align}
\mathbb{E}\bigg[\frac{|\bar{p}_{t|2}|}{\bar{p}_{t|1}}\cdot\bar{\delta}_t\bigg]^2 &\approx \bigg[1 - \frac{(1-\beta_1)(1-\beta_2)}{p_{t|1}p^2_{t|2}}\sum_{\tau}(\beta_1\beta_2)^{t-\tau} g_{\tau}\sigma^2_{\tau} - \frac{(1-\beta_2)}{2p^2_{t|2}}\sum_{\tau}\beta_2^{t-\tau} \sigma^2_{\tau} + \frac{3(1-\beta_2)^2}{2p^4_{t|2}}\sum_{\tau}(\beta_2^2)^{t-\tau} g_{\tau}^2\sigma^2_{\tau}\bigg]^2 \nonumber \\
&\approx 1 - \frac{2(1-\beta_1)(1-\beta_2)}{p_{t|1}p^2_{t|2}}\sum_{\tau}(\beta_1\beta_2)^{t-\tau} g_{\tau}\sigma^2_{\tau} - \frac{(1-\beta_2)}{p^2_{t|2}}\sum_{\tau}\beta_2^{t-\tau} \sigma^2_{\tau} + \frac{3(1-\beta_2)^2}{p^4_{t|2}}\sum_{\tau}(\beta_2^2)^{t-\tau} g_{\tau}^2\sigma^2_{\tau},\nonumber \\
\mathbb{E}\bigg[\bigg(\frac{|\bar{p}_{t|2}|}{\bar{p}_{t|1}}\cdot\bar{\delta}_t\bigg)^2\bigg]&\approx 1 - \frac{4(1-\beta_1)(1-\beta_2)}{p_{t|1}p^2_{t|2}}\sum_{\tau}(\beta_1\beta_2)^{t-\tau} g_{\tau}\sigma^2_{\tau} - \frac{(1-\beta_2)}{p^2_{t|2}}\sum_{\tau}\beta_2^{t-\tau} \sigma^2_{\tau} + \frac{4(1-\beta_2)^2}{p^4_{t|2}}\sum_{\tau}(\beta_2^2)^{t-\tau} g_{\tau}^2\sigma^2_{\tau} \nonumber \\
&\quad + \frac{(1-\beta_1)^2}{p^2_{t|1}}\sum_{\tau}(\beta_1^2)^{t-\tau}\sigma^2_{\tau}.
\end{align}
Therefore we obtain the variance term in the following form:
\begin{align}
Var\big(\bar{\delta}_t\big) \approx &
\frac{1}{p^2_{t|2}}\bigg[(1-\beta_1)^2\sum_{\tau}(\beta_1^2)^{t-\tau}\sigma^2_{\tau} 
 + \frac{p_{t|1}^2}{p^2_{t|2}}(1-\beta_2)^2\sum_{\tau}(\beta_2^2)^{t-\tau} \bigg(\frac{g^2_{\tau}}{p^2_{t|2}}\bigg)\sigma^2_{\tau}\nonumber \\
 &\qquad- 2\frac{p_{t|1}}{|p_{t|2}|}(1-\beta_1)(1-\beta_2)\sum_{\tau}(\beta_1\beta_2)^{t-\tau}\bigg(\frac{g_{\tau}}{|p_{t|2}|}\bigg)\sigma^2_{\tau}\bigg] \nonumber \\
 = & 
 \frac{1}{p^2_{t|2}}\bigg[(1-\beta_1)^2\sum_{\tau}(\beta_1^2)^{t-\tau}\sigma^2_{\tau} 
 + \delta^2_t(1-\beta_2)^2\sum_{\tau}(\beta^2_2)^{t-\tau} \bigg(\frac{g^2_{\tau}}{p^2_{t|2}}\bigg)\sigma^2_{\tau}\nonumber \\
 &\qquad- 2\delta_t(1-\beta_1)(1-\beta_2)\sum_{\tau}(\beta_1\beta_2)^{t-\tau}\bigg(\frac{g_{\tau}}{|p_{t|2}|}\bigg)\sigma^2_{\tau}\bigg].
\end{align}
On the one hand, this formula can be directly used in the algorithm. On the other hand, however, the two summed terms: $\frac{g_{\tau}}{|p_{t|2}|}(\beta_1\beta_2)^{t-\tau}\sigma^2_{\tau} $ and $ \frac{g^2_{\tau}}{p^2_{t|2}}(\beta_2)^{t-\tau}\sigma^2_{\tau}$, may cost extra calculations. And to simplify the implementation, one can focus on the orders of $\bar{g}_t$, and approximate the above expression as:
\begin{align}
Var\big(\bar{\delta}_t\big) \approx & 
 \frac{1}{p^2_{t|2}}\bigg[
 \sigma_{t|1}^2 - 2\delta_t\langle\epsilon_{t|1}\epsilon_{t|2}\rangle + \delta_t^2\sigma_{t|2}^2
 \bigg],
\end{align}
where we defined the following variance-like terms:
\begin{align}
\sigma_{t|1}^2 &= (1-\beta_1)^2\sum_{\tau}(\beta_1^2)^{t-\tau}\sigma^2_{\tau}, \nonumber \\
\sigma_{t|2}^2 &= (1-\beta_2)^2\sum_{\tau}(\beta^2_2)^{t-\tau} \sigma^2_{\tau}, \nonumber \\
\langle\epsilon_{t|1}\epsilon_{t|2}\rangle &= (1-\beta_1)(1-\beta_2)\sum_{\tau}(\beta_1\beta_2)^{t-\tau}\sigma^2_{\tau}.
\end{align}
The three above expressions can be easily calculated as the momentum of the variance given the series $\{\sigma_{\tau}^2\}_{\tau=1}^t$, as long as we take the decaying parameter as $\beta = \{\beta_1^2, \beta_2^2, , \beta_1\beta_2\}$ respectively:
\begin{align}
u_t &= \beta_1^2 u_{t-1} + (1-\beta_1^2)\sigma_t, \nonumber \\
v_t &= \beta_1^2 v_{t-1} + (1-\beta_2^2)\sigma_t, \nonumber \\
w_t &= \beta_1\beta_2 w_{t-1} + (1-\beta_1\beta_2)\sigma_t.
\end{align}
Then we can obtain the variance-like terms as:
\begin{align}
\sigma_{t|1}^2 &= \frac{(1-\beta_1)^2}{(1-\beta_1^2)}u_t, \nonumber \\
\sigma_{t|2}^2 &= \frac{(1-\beta_2)^2}{(1-\beta_2^2)}v_t, \nonumber \\
\langle\epsilon_{t|1}\epsilon_{t|2}\rangle &= \frac{(1-\beta_1)(1-\beta_2)}{(1-\beta_1\beta_2)}w_t,
\end{align}
which is the implemented algorithm mentioned in the main text. And in practice, compared with conventional Adam, these three terms are the only extra terms required to be calculated.

Note that the above calculation, with/without approximations, can be easily extended to other current algorithms, e.g. Momentum-SGD, AdaGrad, RMSprop and so on.

\end{document}